\documentclass[journal]{IEEEtran}
\ifCLASSINFOpdf
\else
\fi
\usepackage{amsmath}
\usepackage{amsfonts}
\usepackage{amsthm}
\usepackage{booktabs}
\usepackage{multirow} 
\usepackage{adjustbox}

\newtheorem{theorem}{Theorem}
\newtheorem*{thm}{Theorem}

\newtheorem{definition}{Definition}

\DeclareMathOperator{\spn}{span}

\begin{document}
%
\title{Rational Gaussian wavelets and corresponding model driven neural networks$^*$}
%
%
%

\author{Attila Mikl\'os \'Amon$^1$,
        Kristian Fenech$^2$, P\'eter Kov\'acs$^1$ and Tam\'as D\'ozsa$^{3,1,4}$
\thanks{$^*$ $\copyright$ 2025 IEEE. Personal use of this material is permitted. Permission from IEEE must be obtained for all other uses, in any current or future media, including reprinting/republishing this material for advertising or promotional purposes, creating new collective works, for resale or redistribution to servers or lists, or reuse of any copyrighted component of this work in other works.}
\thanks{$^1$ E\"otv\"os Lor\'and University, Faculty of Informatics Department of Numerical Analysis. Contact: ze3vjn@inf.elte.hu (Attila \'Amon Mik\'os), kovika@inf.elte.hu (P\'eter Kov\'acs), dotuaai@inf.elte.hu (Tam\'as D\'ozsa)}
\thanks{$^2$ E\"otv\"os Lor\'and University, Faculty of Informatics Department of Artificial Intelligence. Contact: fenech@inf.elte.hu (Kristian Fenech)}
\thanks{$^3$ HUN-REN Institute for Computer Science and Control, Systems and Control Laboratory. Contact: dozsa.tamas@hun-ren.sztaki.hu (Tam\'as D\'ozsa)}
\thanks{$^4$ University of Gy\H{o}r, Vehicle Industry Research Center: dozsa.tamas@hun-ren.sztaki.hu (Tam\'as D\'ozsa)}
}

%
%

\markboth{Journal of \LaTeX\ Class Files,~Vol.~14, No.~8, August~2015}%
{Shell \MakeLowercase{\textit{et al.}}: Bare Demo of IEEEtran.cls for IEEE Journals}
%



\maketitle

\begin{abstract}
In this paper we consider the continuous wavelet transform using Gaussian wavelets multiplied by an appropriate rational term. The zeros and poles of this rational modifier act as free parameters and their choice highly influences the shape of the mother wavelet. This allows the proposed construction to approximate signals with complex morphology using only a few wavelet coefficients. We show that the proposed rational Gaussian wavelets are admissible and provide numerical approximations of the wavelet coefficients using variable projection operators. In addition, we show how the proposed variable projection based rational Gaussian wavelet transform can be used in neural networks to obtain a highly interpretable feature learning layer. We demonstrate the effectiveness of the proposed scheme through a biomedical application, namely, the detection of ventricular ectopic beats (VEBs) in real ECG measurements.
\end{abstract}


%
\IEEEpeerreviewmaketitle

\section{Introduction}
\IEEEPARstart{R}{epresenting} signals using wavelet coefficients has proved beneficial in a variety of fields. Even though the modern notion of the wavelet transform is usually dated from~\cite{grossmann1984decomposition} written by Alex Grossmann and Jean Morlet, historical accuracy requires that previous works now included in broader wavelet theory are also mentioned. These include results achieved by Haar, Franklin, Littlewood and Paley and others (we recommend~\cite{polikar1999story} for a thorough overview). 

The widespread use of wavelets in various signal and image processing applications is no coincidence. In addition to important mathematical properties, the wavelet transform provides the so-called time-scale representation of signals, which is closely associated with time-frequency representations. In other words, the wavelet transform allows us to study the frequency profile of signals in certain time windows. In this way, representing signals using wavelet transforms inherently provides interpretable information about their behavior. This property has been widely expoited in signal processing related fields such as fault detection, where sudden changes in the frequency profile of signals can indicate the appearance of faults~\cite{isermann2006fault}. Contrary to classical time-frequency approaches such as the Gabor-transform (see e.g.~\cite{gasquet2013fourier}), wavelet representations allow for different time localisation properties at different scales (frequencies), which is also well exploited in applications. Finally, the introduction of multiresolution analysis (see~\cite{mallat1989theory} and e.g.~\cite{gasquet2013fourier}) allowed for a practically realizable way to construct orthogonal wavelets. In turn, quick algorithms realizing the discrete wavelet transform were developed~\cite{sweldens1998lifting} allowing for many applications in signal and image processing~\cite{rajani2023r, isermann2006fault, fan2023lwtnet, tian2023multi}. For a more comprehensive review of wavelets and their history, the authors recommend~\cite{daubechies1992ten}.

In most applications, a fixed so-called mother (or analyzing) wavelet is used to perform wavelet analysis of the signals. These may have certain general beneficial properties such as a compact support, symmetry or smoothness. In addition, a mother wavelet may be selected due to its morphological similarity to the signals to be analyzed, for example so-called Ricker wavelets are frequently employed to represent ECG signals~\cite{li2021waveletkernelnet, burke2001ecg}. In practical applications however, signals are always represented by a finite set of wavelet coefficients and fixing a mother wavelet may not be able to capture the behavior of the signal to a satisfactory degree. In order to overcome this limitation, a variety of adaptive wavelet transforms have been introduced (see e.g.~\cite{claypoole1998adaptive, erdol1996wavelet, chen2006adaptive}). For example in~\cite{claypoole1998adaptive, erdol1996wavelet}, various adaptive discrete wavelet transforms are discussed with the purpose to approximate the signals well using only a few wavelet coefficients. These methods however are difficult to use in many real life scenarios, because the discrete wavelet coefficients lack the time invariant property and the resulting mother wavelets often do not possess desirable properties such as smoothness and symmetry. For this reason, adaptive continuous wavelet transforms were also investigated. In~\cite{chen2006adaptive} for example, the center frequency parameter of Morlet wavelets was subject to optimization to match the behavior of the processed signals.

In this work, we present a new class of mother wavelets referred to as rational Gaussian wavelets (RGW). Instead of previous approaches, where parameters of existing wavelets were optimized (e.g.~\cite{chen2006adaptive}), the proposed wavelet class has arbitrary degrees of freedom. This allows for the construction of highly unusual mother wavelet morphologies (see e.g. Fig.~\ref{fig:ecg}), while guaranteeing smoothness and symmetry. Admissibility is proved for the RGW class. In addition, every parameter of the proposed class of wavelets can be optimized using gradient based methods which allows the proposed transform to be easily included in modern machine learning (ML) architectures.

ML methods have become increasingly important in signal processing~\cite{sahoo2020machine, dahrouj2021overview}. This is in large part due to their ability to approximate highly nonlinear operators based solely on data. Deep learning methods have achieved remarkable results in image processing~\cite{gu2018recent}, natural language processing~\cite{fanni2023natural} and biomedical signal processing~\cite{kovacs2022vpnet, vpsvm, luz2016ecg} among other disciplines. One of the most cited limitations of ML methods however is that well performing models usually consist of a large number of parameters, which are difficult to interpret by humans. For example, the popular GPT-3 model contains $\approx 175$ billion parameters~\cite{gpt}. These factors make the deployment of such models expensive in the case of real time, and especially safety critical applications (such as autonomous driving and biomedical signal processing). 

To overcome this issue, several recent research directions have been considered. The taxonomy of explainable artificial intelligence methods proposed in~\cite{Lipton} identifies two important, distinct ways to approach the question of explainability. So-called post-hoc methods aim to provide explanations for the behavior of large deep learning models (see e.g.~\cite{Lipton} and~\cite{spricher}). Transparent ML approaches on the other hand try to propose partially interpretable ML models. Recent examples of transparent ML schemes include physics informed neural networks~\cite{hammernik2023physics, chen2018neural, pereira2020fema} which incorporate classical physical models with meaningful parameters into ML models. Another advancement in transparent machine learning was the development of so-called deep unfolding networks~\cite{hershey2014deep}, which implement the iterative steps of classical numerical methods as layers of a neural network. It is important to mention that the use of the wavelet transform has also been considered previously as a means of achieving transparent ML models. In particular, the recently introduced WaveletKernelNet~\cite{li2021waveletkernelnet} (WKN) exploits the relationship between convolution and wavelet coefficients to create a partially interpretable neural network relying on the wavelet transform. Nevertheless in section~\ref{sec:vp} we provide a qualitative comparison between the chosen model architecture and WKN.

In this work, of particular interest to us are so-called variable projection (VP) based ML methods~\cite{kovacs2022vpnet, golub1973differentiation, golub2003separable}. VP operators implement adaptive orthogonal projections in Hilbert spaces and can be used to obtain interpretable and sparse signal representations. They have been successfully integrated into various ML models including neural networks~\cite{kovacs2022vpnet} and support vector machines~\cite{vpsvm}. In our recent work~\cite{eusipco}, we proposed the approximation of continuous wavelet coefficients using VP operators. In fact, a VP-neural network architecture was considered, where the first layer was responsible for representing the input using only a few wavelet coefficients. The learnable parameters of the layer consisted of the scale and translation parameters associated with the wavelet coefficients. The output of the layer was the collection of wavelet coefficients defined by these learned parameters. In this way, the proposed model learned which time and scale parameters provided the most crucial information for the classification of the input signals. In the current study, this idea is expanded upon by introducing a variable projection based continuous wavelet transform which uses the rational Gaussian wavelet class that we introduce in this paper. In this approach, not only translation and scale parameters, but also parameters governing the morphology of the mother wavelet can be learned. As shown in a case study in section~\ref{sec:expr}, this allows for sparser signal representations and clearer explainability than similar alternative ML methods.

In summary, the main contributions of our paper are the following:

\begin{enumerate}
    \item The introduction of rational Gaussian wavelets (RGW). Through an arbitrary number of parameters, these wavelets allow for the manipulation of the mother wavelet's morphology. 
    \item We show that the RGW class is admissible, which allows for signal reconstruction based on the continuous wavelet coefficients (in $L_2(\mathbb{R})$ norm).
    \item We show that parameters of an RGW mother wavelet can be optimized using gradient based methods. This allows for simple inclusion into the variable projection framework and subsequently ML models utilizing backpropagation.
    \item We provide an upper bound on the error of wavelet coefficient estimates acquired using variable projection operators for signals with a compact support.
    \item We provide a case study on arrhythmia detection from ECG signals. We show that the proposed RGW-VP Net architecture can be used to successfully classify ventricular ectopic heartbeats (VEBs) from ECG data. The interpretability of the learned parameters is examined, and it is shown that the model learns information which corresponds to VEB descriptions in the medical literature. In addition, it is shown that the proposed model requires fewer parameters to achieve state-of-the-art classification accuracy. 
\end{enumerate}

The rest of the paper is organized as follows. In section~\ref{sec:wav} a brief overview of adaptive wavelet transforms and related literature is provided to motivate the introduction of RGW. In section~\ref{sec:rgw} we introduce rational Gaussian wavelets and discuss their most important properties. Section~\ref{sec:vp} reviews variable projection operators and variable projection networks (VP-NET). A realization of the VP-NET architecture implementing a continuous wavelet transform using RGW is also introduced. Section~\ref{sec:expr} contains our experiments and related discussions about ECG classification. Finally, in section~\ref{sec:conc} we summarize our findings and discuss future research directions. Proofs and related calculations can be found in the appendix.

\section{Motivation for adaptive continuous wavelet transforms}
\label{sec:wav}

In this section we discuss wavelet transforms and provide motivation for the introduction of the proposed rational Gaussian wavelets. 

\subsection{Wavelet transforms}

Consider the family of functions

\begin{equation*}
    \psi_{\lambda, \tau}(t) = \frac{1}{\sqrt{ | \lambda |}} \psi( 1/ \lambda \cdot (t - \tau) ), \quad (t, \lambda, \tau \in \mathbb{R}, \lambda \neq 0),
\end{equation*}
where the function $\psi$ is referred to as the "mother" or "analyzing" wavelet. The wavelet transform of a function $f$ with respect to $\psi$ is defined as a function of two variables
\begin{equation}
    \label{eq:wavcoeff}
    W_{\psi} f(\lambda, \tau) := \int_{-\infty}^{\infty} f(t) \overline{\psi}_{\lambda, \tau}(t) dt,
\end{equation}
whenever the integral exists. In Eq.~\eqref{eq:wavcoeff}, $\overline{\psi}$ denotes the complex conjugate of $\psi$. We note that the existence criteria is fulfilled for example if $f, \psi \in L_{2}(\mathbb{R})$ by H\"older's inequality. For a fixed scale-translation pair $(\lambda, \tau) \in \mathbb{R} \setminus \{0\} \times \mathbb{R}$, the number $W_{\psi} f(\lambda, \tau)$ will also be referred to as a \textit{wavelet coefficient} corresponding to $(\lambda, \tau)$.

The admissibility theorem (see e.g.~\cite{gasquet2013fourier}) describes a set of conditions on $\psi$ which allow for signal reconstruction in $L_2$ norm.

\begin{theorem}[Admissibility property]
\label{thm:admis}
Suppose the wavelet $\psi \in L_1(\mathbb{R}) \cap L_2(\mathbb{R})$ satisfies

\begin{enumerate}
    \item $\int_{-\infty}^{\infty} \frac{1}{|\xi|} \cdot \left| \widehat{\psi}(\xi) \right|^2 d \xi = M < \infty $,
    \item $\|\psi\|_2 = 1$,
\end{enumerate}
where $\widehat{\psi}$ denotes the Fourier transform of $\psi$. Then, for any $f \in L_2(\mathbb{R})$ the following results hold:
\begin{enumerate}
    \item Energy conservation:
    \begin{equation*}
        \frac{1}{M} \int_{-\infty}^{\infty} \int_{-\infty}^{\infty} |W_{\psi} f(\lambda, \tau)|^2 \frac{d \lambda d \tau}{\lambda^2} = \int_{-\infty}^{\infty} |f(t)|^2 dt.
    \end{equation*}
    \item Define the function $f_{\varepsilon}$ as
    \begin{equation*}
        f_{\varepsilon}(t) := \frac{1}{M} \int_{|\lambda| \geq \varepsilon} \int_{-\infty}^{\infty} W_{\psi}f(\lambda, \tau) \psi_{\lambda, \tau}(t) \frac{d \lambda d \tau}{\lambda^2}.
    \end{equation*}
    Then, $f$ can be reconstructed from the wavelet coefficients $W_{\psi}f$ in the sense that $f_{\varepsilon} \to f$ in $L_2(\mathbb{R})$ as $\varepsilon \to 0^{+}$.
\end{enumerate}
\end{theorem}

In section~\ref{sec:wav} we state a theorem that guarantees admissibility for the proposed family of wavelets. For further discussion of reconstruction formulas and a proof of theorem~\ref{thm:admis}, we refer to~\cite{gasquet2013fourier, heil2006fundamental}.

In applications, wavelet coefficients are computed in a numerical fashion. This also means that the parameters $(\lambda, \tau)$ have to be restricted to a (finite) subset of $\mathbb{R} \setminus \{0\} \times \mathbb{R}$. A general way to express this (see e.g.~\cite{gasquet2013fourier}) is to consider the wavelet coefficients $W_{\psi}f(\lambda_n, \tau_m) \ (m,n \in \mathbb{Z})$, where 
\begin{equation}
    \label{eq:discretewavpoints}
    \lambda_n := \alpha^{-n}, \quad \tau_m := m \beta.
\end{equation}
In this case, the coefficients are calculated using the functions $\psi_{\lambda_n, \tau_m}(t) = \alpha^{n/2} \psi(\alpha^n \cdot t - m \cdot \beta)$. The computational cost of the wavelet transform increases as $\alpha$ tends to $1$ and $\beta$ tends to $0$. Usually, the parameters are chosen as $\alpha=2$ and $\beta=1$, a choice, where the scales $\lambda_n$ correspond to octaves in music~\cite{gasquet2013fourier}. As mentioned in the introduction, different strategies have been developed to calculate the required wavelet coefficients $W_{\psi} f(\lambda_n, \tau_m)$. 

A popular approach is to construct the family $\psi_{\lambda_n, \tau_m}$ so that they form an orthogonal basis in $L_2(\mathbb{R})$ using multiresolution analysis. The so-called discrete wavelet transform (DWT) algorithm uses such orthogonal wavelets and has the benefit, that discrete wavelet coefficients can be used to achieve perfect signal reconstruction.
Even though fast DWT methods~\cite{sweldens1998lifting} have been developed relying on digital filtering, the use of orthogonal wavelets has its disadvantages. One problem is that, for some signal analysis tasks, a finer resolution of the coefficients is desirable~\cite{jordan1997implementation}. In addition, DWT algorithms are not translation invariant by default which makes their use difficult for some signal processing tasks. It should be noted that certain translation invariant (and from the reconstruction point of view redundant) variations, such as the stationary wavelet transform,~\cite{nason1995stationary} exist. Finally, when constructing orthogonal wavelets there is a trade off between "desirable" properties. For example, Daubechies' wavelets are $r$ times differentiable, however their support increases with $r$ meaning that smoother mother wavelets become less well localised in time. Another example is that the often desired symmetry property can only be approximated for orthogonal wavelets.

For the above reasons, continuous wavelet transform (CWT) algorithms are also frequently used for signal analysis~\cite{jordan1997implementation}. These approximate the wavelet coefficients~\eqref{eq:wavcoeff} using numerical quadrature formulas. CWT algorithms cannot achieve perfect signal reconstruction. On the other hand, CWT algorithms retain the important translation invariance property, and can be used with mother wavelets which were not constructed using multiresolution analysis. This allows for the use of wavelets with desirable properties such as Gaussian wavelets~\cite{ososkov2000gaussian}, which can be acquired by differentiating the Gaussian function. For example, Ricker's wavelet has enjoyed popularity for biological signal processing and fault detection applications~\cite{li2021waveletkernelnet, burke2001ecg}. The adaptive wavelet family proposed in section~\ref{sec:rgw} is also closely connected to Gaussian wavelets.

\subsection{Adaptive wavelet transforms}

Choosing the correct analyzing wavelet for the CWT algorithm can be a challenging problem. Fig.~\ref{fig:wavcomp} shows the scalograms (absolute values of wavelet coefficients) computed by MatLab's CWT method for a non-stationary signal consisting of different sinusoids in different time intervals. For this example, the analyzing wavelets were chosen as a complex Morlet and a Ricker wavelet. The analyzing wavelets which were used to create the scalorgrams on Fig.~\ref{fig:wavcomp} are illustrated in Fig.~\ref{fig:cmor}. As can be seen, different choices of the analyzing wavelet produce different quality time-scale representations. This motivates the development of adaptive analyzing wavelets, whose morphologies depend on tunable parameters.
\begin{figure}[!htb]
    \centering
    \includegraphics[width=\linewidth]{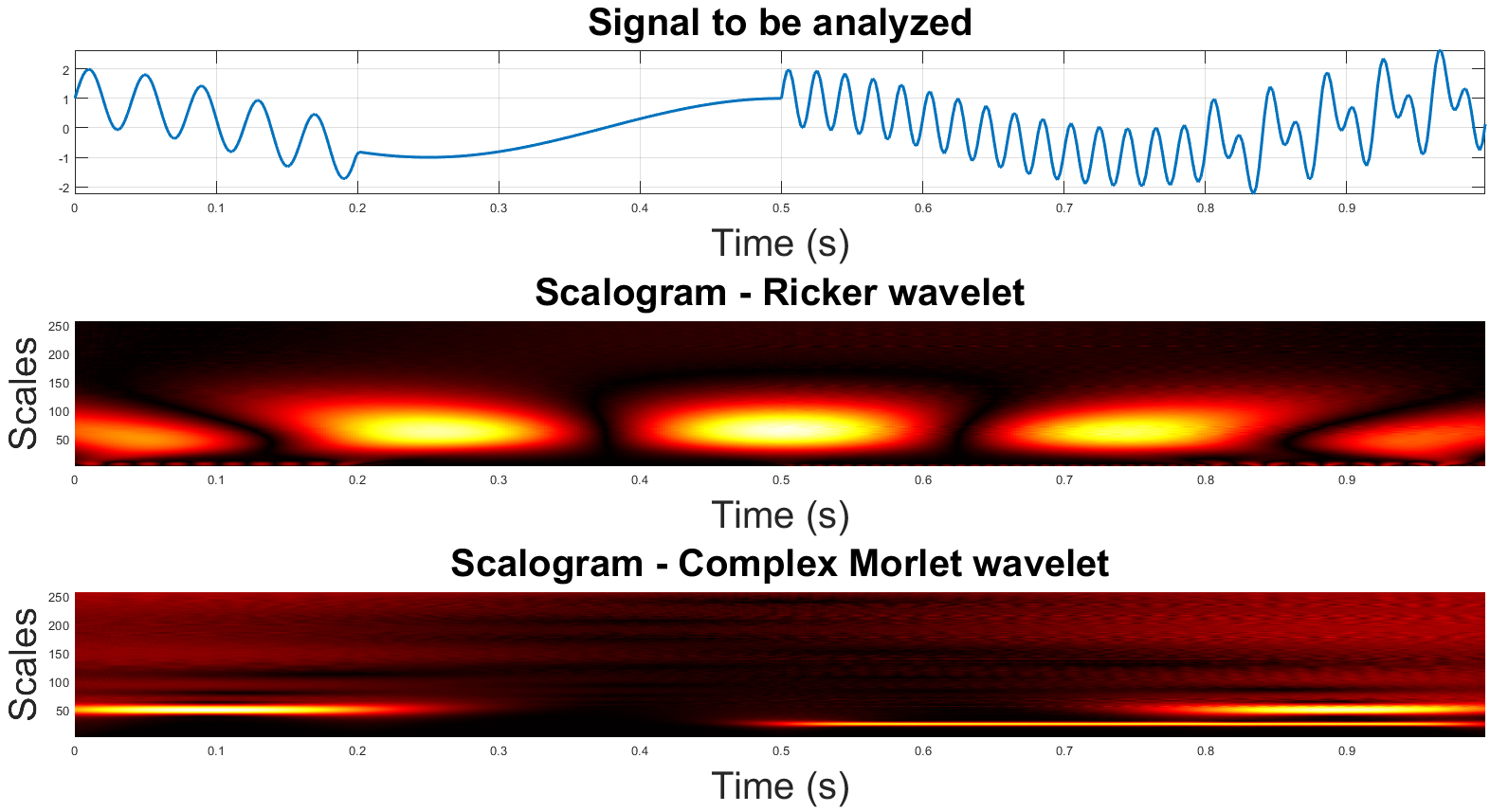}
    \caption{The frequency components of the signal change as a function of time. Scalograms returned by MatLab's CWT routine using different analyzing wavelets. The scalogram in the bottom row captures the frequency profile of the signal much more clearly, demonstrating that the quality of the information provided by wavelet coefficients heavily depends on the choice of the analyzing wavelet.}
    \label{fig:wavcomp}
\end{figure}
\begin{figure}[!htb]
    \centering
    \includegraphics[width=\linewidth]{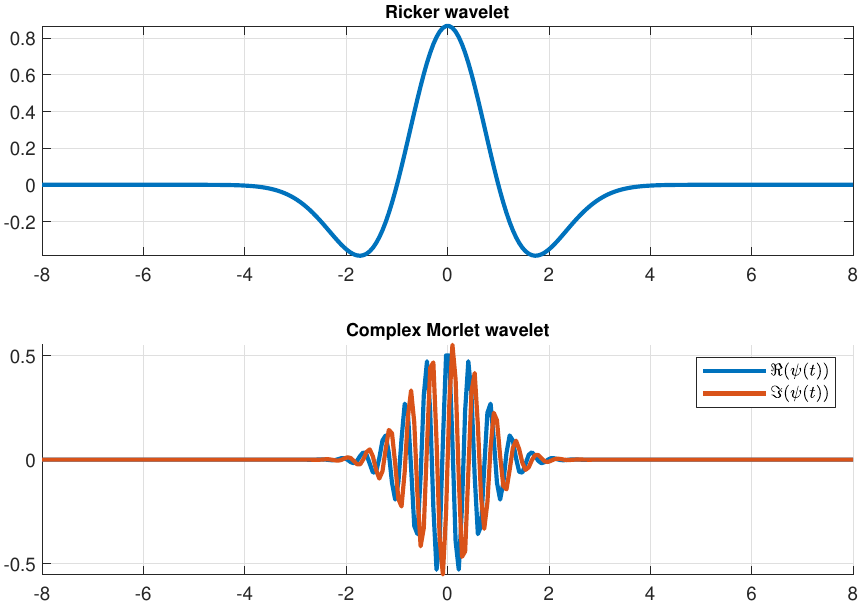}
    \caption{The analyzing wavelets used to construct Fig.~\ref{fig:wavcomp}.}
    \label{fig:cmor}
\end{figure}
The idea of adaptive wavelet transforms is not new. Efforts aimed at the development of adaptive DWT methods usually focused on finding appropriate wavelet filter values to match the signal. In~\cite{claypoole1998adaptive} for example, the so-called lifting scheme (see e.g.~\cite{sweldens1998lifting}) is utilized to construct filters which match the signal morphology. The lifting algorithm is viewed as a "prediction-error" decomposition, where the approximation coefficients of the DWT transform at a given scale are viewed as "predictors", and the detail coefficients are considered the error of the prediction. Minimizing the prediction error at each scale level allows for better quality reconstruction using only coefficients belonging to the first few scale levels. We also point the reader towards~\cite{erdol1996wavelet} for other DWT based adaptive wavelet decomposition schemes. Even though these are usually quick and guarantee that the transformation can be inverted, they suffer from the usual drawbacks associated with DWT. Namely, the smoothness and and symmetry properties of the learned wavelets cannot be guaranteed. In addition, for most of these algorithms time invariance of the computed wavelet coefficients is also not fulfilled, which makes their application difficult for the analysis of signals. Finally, DWT algorithms compute approximation and detail coefficients corresponding to the discretization shown in Eq.~\eqref{eq:discretewavpoints}. If instead the computed dilation and translation coefficients were subject to optimization, sparser signal representations may be achieved as shown in~\cite{eusipco}.

For the above reasons, adaptive CWT algorithms also enjoyed growing attention. Several papers optimized the parameters of Morlet wavelets~\cite{chen2006adaptive, lin2003gearbox} for various fault-detection applications. An important benefit of adaptive CWT schemes, is that they adapt the morphology of the analyzing wavelet to match the input signals, while simultaneously ensuring smoothness and symmetry properties. In~\cite{chen2006adaptive} for example, properties associated with Morlet wavelets are retained even though the centre frequency parameter of the analyzing wavelet is modified through numerical optimization. Although adaptive CWT methods do not guarantee perfect signal reconstruction, in applications (e.g. fault detection) invertibility is not the most crucial property of the transformation. In fact, in these applications a sparse representation of the input signals is usually preferred, which allows for clearer distinction between fault classes~\cite{isermann2006fault}. CWT is ideal for this, since in addition to learning parameters of the analyzing wavelet, adaptive CWT algorithms can also treat the dilation and translation parameters of the computed wavelet coefficients as free parameters. This idea was utilized for example in~\cite{li2021waveletkernelnet} and~\cite{eusipco}, where different CWT algorithms were used as feature extraction layers in neural networks. In this paper, we also demonstrate the utility of the proposed wavelet transformation using such wavelet transform enchanced machine learning models (see sections~\ref{sec:vp} and~\ref{sec:expr}). 
 
\section{Rational Gaussian Wavelets}
\label{sec:rgw}

Gaussian wavelets are defined by
\begin{equation}
    \label{eq:gausswav}
    \psi_n(t) := C_n \cdot \frac{d^n }{d t^n} e^{-t^2/2},
\end{equation}
where $n \in \mathbb{N}$ and the constant satisfies $C_n = 1 / \left\| \frac{d^n}{d t^n} e^{-t^2/2} \right\|_2$. These types of wavelets are commonly used in signal and image processing, especially since they exhibit good localisation properties on both time and frequency space. In particular, Ricker wavelets (the $n=2$ case in Eq.~\eqref{eq:gausswav}) have been successfully applied for fault detection~\cite{li2021waveletkernelnet}, and biological signal processing~\cite{burke2001ecg}.

In this paper, we consider the case $n=1$ and multiply the Gaussian $e^{-t^2/2}$ with an appropriate rational function, whose poles and zeros will act as free parameters in our construction. If the rational term satisfies the following criteria, the resulting family of wavelets will satisfy the conditions of theorem~\ref{thm:admis}. Consider first, the polynomials
\begin{equation}
\label{eq:baspols}
\begin{split}
    & r_{z}(t) := (t - z) (t + z)( t - \tilde{z}) (t + \tilde{z}) \\ &  (z \in \mathbb{C}, \ \Im(z) \neq 0, \ t \in \mathbb{R}),
\end{split}
\end{equation}
\noindent where $\Im(z)$ denotes the imaginary part of the complex parameter $z$ and $\tilde{z} = - \Re(z) + i \Im(z)$. We note that polynomials of the form $r_z(x) = (x - z)(x - \overline{z})$ may also be used if $\Re{z} = 0$ for the proposed construction. We denote the class of polynomials satisfying Eq.~\eqref{eq:baspols} by $\mathcal{P}_{e}$. Clearly, the elements of $\mathcal{P}_{e}$ are real valued even functions, who do not vanish on the real line. Using the class $\mathcal{P}_{e}$ we can introduce the class of rational functions
\begin{equation}
\label{eq:ratclass}
    \mathcal{V} := \left \{ v(t) = \frac{1}{\prod_{k=0}^{n-1} q_k(t)} : \ q_k \in \mathcal{P}_e, \ n \in \mathbb{N} \right\}.
\end{equation}
Then, the rational functions $v \in \mathcal{V}$ are also real and do not vanish anywhere on $\mathbb{R}$. Consider now the polynomial
\begin{equation}
    \label{eq:polpart}
    \begin{split}
    & P(t) := t \cdot \prod_{k=1}^{p} (t - t_k) ( t + t_k) \\ & ( t_k \in \mathbb{C} \setminus \{0 \}, \ p \in \mathbb{N}).    
    \end{split}
\end{equation}
With the conditions stated in Eq.~\eqref{eq:polpart}, $P$ is an odd polynomial satisfying $P(t) = -P(-t) \quad (t \in \mathbb{R})$. Now, it is possible to state the following definition. We note that for the biomedical application detailed in this work, we were particularly interested in real-valued wavelets (see section~\ref{sec:expr}). For this reason, and for simplicity we assumed $\Im(t_k) = 0$ in our experiments. It should also be noted, that complex variations of rational Gaussian wavelets are possible, as long as $P$ remains an odd function and $\mathcal{V}$ contains even functions, who do not vanish on the real line (see the proof of admissibility in the appendix).

\begin{definition}[Rational Gaussian wavelet]
The function $\Psi \in L_2(\mathbb{R})$ is called a rational Gaussian wavelet (RGW), if
\begin{equation}
    \label{eq:RGW}
    \psi^{\boldsymbol{\eta}}(t) := C(\boldsymbol{\eta}) \cdot P^{\boldsymbol{\eta}}(t) \cdot v^{\boldsymbol{\eta}}(t) \cdot e^{-t^2/2}, \ (t \in \mathbb{R}, \ \boldsymbol{\eta} \in \mathbb{C}^{p+n})
\end{equation}
where $P^{\boldsymbol{\eta}}$ is a polynomial satisfying Eq.~\eqref{eq:polpart}, $v^{\boldsymbol{\eta}} \in \mathcal{V}$ and $\boldsymbol{\eta}$ contains the poles $z$ of $v^{\boldsymbol{\eta}}$ and the non-zero roots of $P^{\boldsymbol{\eta}}$. Moreover, the constant $C(\boldsymbol{\eta}) \in \mathbb{R}$ is chosen such that $\|\psi^{\boldsymbol{\eta}} \|_{2} = 1$ holds. 
\end{definition}
Mother wavelets defined according to Eq.~\eqref{eq:RGW} clearly belong to $L_2(\mathbb{R})$, in fact the following theorem holds.
\begin{theorem}[Admissibility of rational Gaussian wavelets]
\label{thm:admisrgw}
    Suppose $\psi \in L_2(\mathbb{R})$ is defined according to Eq.~\eqref{eq:RGW}. Then, there exists a number $M \in \mathbb{R}_+$, such that
    \begin{equation*}
        \int_{-\infty}^{\infty} \frac{1}{|\xi|} \left| \widehat{\psi}(\xi) \right|^2 d\xi = M
    \end{equation*}
    holds.
\end{theorem}
The proof of theorem~\ref{thm:admisrgw} can be found in the appendix. The admissibility property guarantees that the proposed RGW wavelet transform can be inverted with respect to the $L_2$ norm. It should be noted, that theorem~\ref{thm:admisrgw} also holds, if we replace the set of possible rational terms $\mathcal{V}$ with the linear space
\begin{equation*}
    \mathcal{W} := \left\{ \sum_{j=0}^{k} c_k \cdot v_j(x) : \ v_j \in \mathcal{V}, \ k \in \mathbb{N} \right\},
\end{equation*}
however, for simplicity we state the results using the definition above. As seen in the appendix, the proof relies on two properties of the RGWs: symmetry (they are odd functions), and the fact that the denominator of $v$ in Eq.~\eqref{eq:RGW} does not dissappear anywhere. These properties however are also satisfied if $v$ is chosen from $\mathcal{W}$ instead of $\mathcal{V}$.

RGW mother wavelets depend on a number of free parameters. The modification of these parameters greatly influences the morphology of the analyzing wavelet. For example, moving the poles of $v^{\boldsymbol{\eta}}$ close to the real line leads to the appearance of sharp peaks in the mother wavelet. This spike-like morphology resembles a number of real life signals which appear in fault detection and biomedical applications (see e.g. Fig.~\ref{fig:ecg}). 

It is important to note, that despite the sharp oscillations shown in Fig.~\ref{fig:ecg}, RGWs retain smoothness and symmetry. This adaptive property of the proposed wavelets allows for the capture of important signal behavior using only a few wavelet coefficients. That is, a sparse wavelet coefficient representation of a large class of signals may be obtained by optimizing the poles of $v^{\boldsymbol{\eta}}$ and the zeros of $P^{\boldsymbol{\eta}}$ with respect to an objective function. This makes the continuous wavelet transform utilizing RGWs a powerful trainable feature extractor. Since RGWs are also analytic with respect to the poles of $v^{\boldsymbol{\eta}}$ and $P^{\boldsymbol{\eta}}$, gradient based methods (such as backpropagation algorithms) may be used to optimize these parameters. This allows for an easy integration of the scheme into existing machine learning models such as neural networks.

\subsection{Wavelet transform based machine learning models}

In this section, we discuss methods to extend existing machine learning models with the proposed RGW based continuous wavelet transform. Feature extraction schemes relying on the wavelet transform have been studied before (see e.g.~\cite{burke2001ecg, isermann2006fault, dozsa2019ensemble, eusipco, li2021waveletkernelnet}). In many works, the wavelet transform (discrete or continuous) is used as a static feature extractor~\cite{burke2001ecg, isermann2006fault, dozsa2019ensemble}. That is, a wavelet coefficient representation of the input signals is obtained, which is then passed to a machine learning algorithm to solve a regression or classification problem. The feature extraction step in this case is conducted separately from the training of the ML model. Since wavelet coefficients represent information about the signals' time-frequency profile, the magnitude of coefficients corresponding to fixed scale and translation parameters conveys interpretable information about signal behavior at given times and frequencies. 

In most CWT based machine learning methods however, a large number of wavelet coefficients is computed, then passed to the underlying machine learning algorithms. A more adaptive approach was proposed in~\cite{li2021waveletkernelnet}, where the following observation was made. For a fixed scale parameter $\lambda \in \mathbb{R} \setminus \{0\}$, the function $W_{\psi}f(\lambda, \cdot) : \mathbb{R} \to \mathbb{C}$ can be expressed by
\begin{equation*}
    W_{\psi}f(\lambda, \cdot) = \left(f \ast \psi_{\lambda, 0} \right) (\tau) = \int_{-\infty}^{\infty} f(t) \overline{\psi}(1/\lambda \cdot (t - \tau)) dt.
\end{equation*}
In other words, wavelet coefficients can be computed at fixed scales using convolution with the kernel $\psi_{\lambda, 0}$. Exploiting this observation, in~\cite{li2021waveletkernelnet} the use of convolution layers in neural networks with scaled wavelet kernels was proposed. The learnable parameters of these wavelet kernel convolution layers were the scales $\lambda_k$ and initial translation $\tau_k \quad (k=1,\ldots,M)$. Then, wavelet coefficients were approximated by computing the (discrete) convolutions
\begin{equation}
    \label{eq:wavekernel}
    \boldsymbol{f} \ast \boldsymbol{\psi_{\lambda_k, \tau_k}} \quad (k=1,\ldots,M),
\end{equation}
where $\boldsymbol{f} \in \mathbb{R}^N$ and $\boldsymbol{\psi_{\lambda_k, \tau_k}} \in \mathbb{R}^N$ denote $N$-point equidistant samplings of $f$ and $\psi_{\lambda_k, \tau_k}$ respectively. This approach means, that the input signals are represented with a large number of wavelet coefficients, therefore, in wavelet kernel net pooling layers are utilized to obtain a sparse signal representation. We note that the proposed RGW wavelets could easily be adapted for use with wavelet kernel net proposed in~\cite{li2021waveletkernelnet} and may perform well for certain applications. We plan to investigate such constructions in the future, however because of the following reason we opted for a variable projection based neural network architecture in this study (see section~\ref{sec:vp}). 

The main benefit of optimizing the zeros and poles of the rational term $P^{\boldsymbol{\eta}}(x) \cdot v^{\boldsymbol{\eta}}(x)$ in Eq.~\eqref{eq:RGW}, is that mother wavelet morphologies similar to the input signals may be obtained. This is beneficial, because in this way, precise reconstructions of the input signals may be obtained using only a few wavelet coefficients (see section~\ref{sec:vp} and Fig.~\ref{fig:ecg}). In this way, the computation of a large number of wavelet coefficients using the wavelet kernel net would be contrary to our efforts. 

\subsection{Variable projection and wavelet coefficients}
\label{sec:vp}

In this section we summarize a variable projection based methodology which allows for the inclusion of the RGW into machine learning, and specifically, neural network models. The discussed methodology was originally developed in~\cite{kovacs2022vpnet} for neural networks and~\cite{vpsvm} for support vector machines (SVMs). In addition, we use of our previous findings in~\cite{eusipco}, where we showed the benefit of approximating wavelet coefficients using variable projection operators. An important contribution of the current study in contrast to~\cite{eusipco}, is that instead of only learning appropriate translation and scale parameters of the wavelet coefficients and considering a fixed mother wavelet, in this work we propose variable projection transformations which can be used to also optimize the poles and zeros in Eq.~\eqref{eq:RGW}. 

In practical applications, signals are usually only available in a discretely sampled form. That is, instead of $f \in L_2(\mathbb{R})$, we may only consider its (equidistant) sampling $\boldsymbol{f} \in \mathbb{R}^N \ (N \in \mathbb{N})$. Consider an $m << N$ dimensional subspace in $\mathbb{R}^N$ spanned by the linearly independent basis vectors $\boldsymbol{u}_1,\ldots,\boldsymbol{u}_m \in \mathbb{R}^N$. Denote this subspace by
\begin{equation*}
    U := \spn \{ \boldsymbol{u}_1,\ldots,\boldsymbol{u}_m \} \subset \mathbb{R}^N
\end{equation*}
and by $\boldsymbol{\widehat{f}} \in U$ the element which minimizes
\begin{equation*}
    \| \boldsymbol{f} - \boldsymbol{u} \|_2^2 \quad (\boldsymbol{u} \in U)
\end{equation*}
It is well known that $\boldsymbol{\widehat{f}}$ exists uniquely and may be expressed by the orthogonal projection
\begin{equation*}
    \boldsymbol{\widehat{f}} = \mathcal{P}_U(\boldsymbol{f}) := \Phi \Phi^{+} \boldsymbol{f},
\end{equation*}
where the columns of $\Phi \in \mathbb{R}^{N \times m}$ coincide with the basis vector $\boldsymbol{u}_k \ (k=1,\ldots,m)$ and $\Phi^{+}$ denotes the Moore-Penrose pseudo inverse~\cite{golub1973differentiation}.  Suppose the basis vectors $\boldsymbol{u}_k$ depend on a real parameter vector $\boldsymbol{\eta} \in \mathbb{R}^q \ (q \in \mathbb{N})$ and thus
\begin{equation*}
    U(\boldsymbol{\eta}) := \spn \{ \boldsymbol{u}_1(\boldsymbol{\eta}), \ldots, \boldsymbol{u}_m(\boldsymbol{\eta}) \}.
\end{equation*}
Then, the projection operator $\mathcal{P}_{U(\boldsymbol{\eta})}$ also depends on $\boldsymbol{\eta}$ and is referred to as a \textit{variable projection operator}. Given a family of subspaces $G(\boldsymbol{\eta})$ spanned by $\boldsymbol{u}_k(\boldsymbol{\eta}) \ (k=1,\ldots,m)$, the best possible approximation (if it exists) of $\boldsymbol{f} \in \mathbb{R}^N$ can be found by minimizing
\begin{equation}
    \label{eq:r2}
    E_2(\boldsymbol{\eta}) := \|\boldsymbol{f} - \widehat{\boldsymbol{f}}(\boldsymbol{\eta}) \|_2^2 = \|\boldsymbol{f} - \Phi(\boldsymbol{\eta}) \Phi(\boldsymbol{\eta})^+ \boldsymbol{f} \|_2^2
\end{equation}
over $\boldsymbol{\eta} \in \mathbb{R}^q$. This nonlinear optimization problem  is usually referred to as a separable nonlinear least squares problem (SNLLS), since for any fixed value of $\boldsymbol{\eta}$, minimizing Eq.~\eqref{eq:r2} degrades to a least squares fitting task. SNLLS problems can be solved using gradient based optimization schemes due to the work of Golub and Pereyra, who in~\cite{golub1973differentiation, golub2003separable} provided an explicit formula for the derivatives $\frac{\partial \Phi(\boldsymbol{\eta})^+}{\partial \boldsymbol{\eta}}$. Using this formula, the gradients of $E_2(\boldsymbol{\eta})$ may be calculated analytically. It is important to note however, that the differentiability of $E_2$ assumes that the partial derivatives $\frac{\partial \Phi(\boldsymbol{\eta})}{\partial \boldsymbol{\eta}}$ exist (see e.g.~\cite{golub1973differentiation}). In other words, the basis functions $g_1(\boldsymbol{\eta}), \ldots, g_m(\boldsymbol{\eta})$ should be differentiable with respect to the parameters $\boldsymbol{\eta}$.

In~\cite{eusipco} we proposed to approximate continuous wavelet coefficients using variable projections. More precisely, let $m \in \mathbb{N}$ denote the number of wavelet coefficients used to represent the signal $\boldsymbol{f}$. Let

\begin{equation*}
    \boldsymbol{\eta} := [\lambda_1, \tau_1, \ldots, \lambda_m, \tau_m] \in \mathbb{R}^{2 \cdot m}
\end{equation*}
and let $\boldsymbol{\psi} \in \mathbb{R}^N$ be an equidistantly sampled analyzing wavelet from $L_2(\mathbb{R})$. Suppose the sampling was done over the effective support (a subset of $\mathbb{R}$ over which $\psi(x)$ significantly differs from $0$) of $\psi$. Define $\Psi(\boldsymbol{\eta}) \in \mathbb{R}^{N \times m}$ as the matrix, whose $k$-th column consists of the sampling of $\psi_{\lambda_k, \tau_k} \ (k=1,\ldots,m)$. Then, one has
\begin{equation*}
    W_{\psi}f(\lambda_k, \tau_k) \approx \left(\Psi(\boldsymbol{\eta})^+ \boldsymbol{f} \right)_k \quad (k=1,\ldots,m).
\end{equation*}
We can provide the following error estimate for the obtained approximation if the input signal $f$ is smooth and compactly supported. These properties can be assumed for a large class of signals encountered in applications (see e.g. section~\ref{sec:expr}). 

\begin{theorem}[Error of variable projection based continuous wavelet coefficients]
\label{them:errorest}
    Let $\psi \in L_2(\mathbb{R})$ be a fixed mother wavelet and suppose $f \in L_2(\mathbb{R})$ is continuously differentiable and compactly supported on $[a, b] \subset \mathbb{R}$. Let $a =: t_0 < t_1 < \ldots < t_{N-1} := b$ be an equidistant sampling of $[a, b]$ with $h := t_{1} - t_0$. In addition, consider $\boldsymbol{\eta} = (\lambda_1, \tau_1, \ldots, \lambda_m, \tau_m) \in \mathbb{R}^{2m}$ with $\lambda_k > 0 \ (k=1,\ldots,m)$. Finally, let $\boldsymbol{f}_k := f(t_k) \ (k=0,\ldots,N-1)$. Then we have
    \begin{equation}
    \label{eq:errest}
        \begin{split}
        & \left|W_{\psi}f(\lambda_k, \tau_k) - h \cdot \left(\Psi(\boldsymbol{\eta})^+ \boldsymbol{f} \right)_k \right| < \\
        & h \cdot \frac{M_1 (b - a)}{2} + h \cdot\|\boldsymbol{f}\|_{\infty} \|\Psi^{*}(\boldsymbol{\eta}) \|_{\infty} \cdot \\ &  \left( \frac{\kappa(G(\boldsymbol{\eta})) }{\|G(\boldsymbol{\eta})\|_{\infty}} + 1 \right),
        \end{split}
    \end{equation}
    where 
    \begin{equation*}
        M_1 := \max_{\xi \in [a, b], k=1,\ldots,m} \left|f'(\xi) \cdot \overline{\psi_{\lambda_k, \tau_k}}(\xi) \right|,
    \end{equation*}
    the columns of $\Psi(\boldsymbol{\eta})$ contain the samplings of $\psi_{\lambda_k, \tau_k} \ (k=1,\ldots,m)$ over $t_j \ (j=0,\ldots,N-1)$ and $G(\boldsymbol{\eta})$ is the Gram-matrix constructed form the columns of $\Psi(\boldsymbol{\eta})$. In Eq.~\eqref{eq:errest}, $\kappa(G(\boldsymbol{\eta}))$ denotes the condition number of $G(\boldsymbol{\eta})$ using the matrix infinity norm and $\Psi^{*}(\boldsymbol{\eta})$ denotes the adjungate of $\Psi(\boldsymbol{\eta})$.         
\end{theorem}

The proof of theorem~\ref{them:errorest} is provided in the appendix.  A consequence of this theorem, is that the proposed approximations tend to the actual wavelet coefficients as the sampling frequency of $f$ increases.

There is a number of arguments which support the approximation of wavelet coefficients using the above-described variable projection based methodology in ML models. Firstly, in contrast to the Wavelet Kernel Net approach, there is no need for pooling and dropout operations to obtain a sparse representation of the input. Furthermore, by minimizing~\eqref{eq:r2}, one obtains wavelet coefficients which are optimal with respect to the reconstruction error in the least squares sense. To demonstrate the effectiveness of calculating wavelet coefficients in this way, Fig.~\ref{fig:ecg} shows the reconstruction of an ECG signal using RGW wavelets (where the variable projection operator is defined as described in section~\ref{subsec:rgwvp}) using $8$ wavelet coefficients.
\begin{figure}[!h]
    \centering
    \includegraphics[trim={0, -3mm, 0, 0}, clip, width=0.49\linewidth]{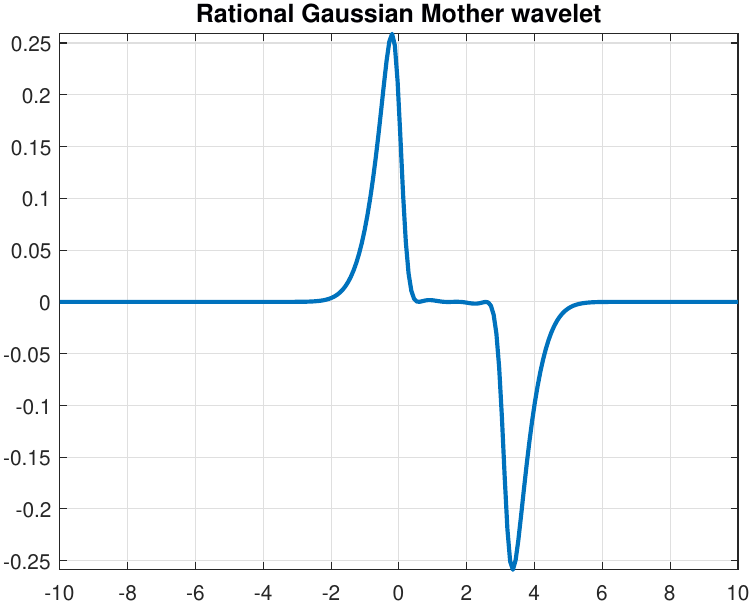}
    \includegraphics[trim={0, 1mm, 0, 0}, clip, width=0.49\linewidth]{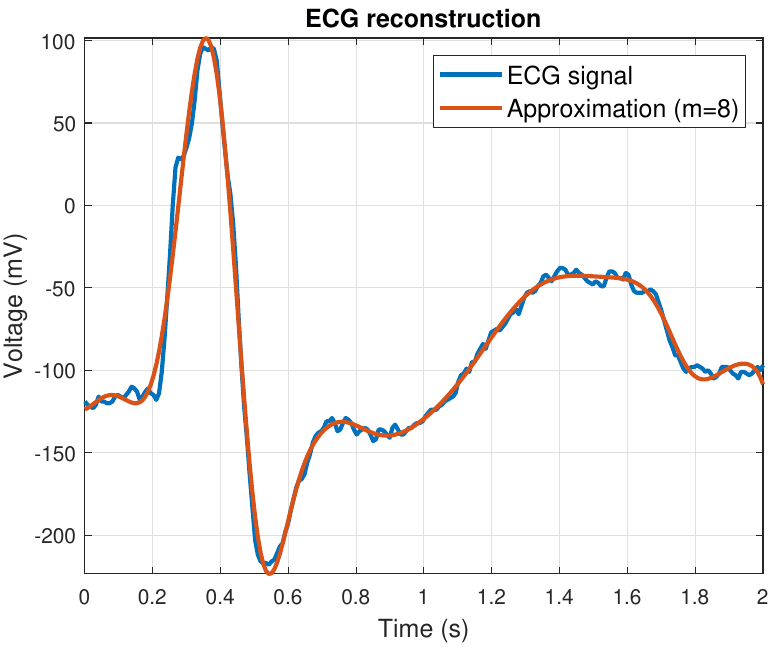}
    \caption{LEFT: A rational Gaussian motherwavelet. Here $n=3$ poles and $p=10$ zeros were used. RIGHT: Reconstruction of ECG signal using $m=8$ wavelet coefficients.}
    \label{fig:ecg}
\end{figure}
Finally, since $E_2(\boldsymbol{\eta})$ may be optimized using gradient based methods, variable projection operators can be included into neural network architectures (and other machine learning models~\cite{vpsvm}) as separate layers and trained using backpropagation schemes. These types of neural networks, often referred to as VP-Net, were introduced in~\cite{kovacs2022vpnet} and have been successfully used with a number of different varaible projection operators to solve environmental recognition~\cite{roadabnorm}, fault detection~\cite{vpsvm} and biomedical signal processing problems~\cite{kovacs2022vpnet, vpsvm}. 

\subsection{The RGW-VP layer}
\label{subsec:rgwvp}

Suppose we would like to solve a classification (or regression) problem using supervised learning. Consider the mappings $\mathcal{P}_{\Psi(\boldsymbol{\eta})} : \mathbb{R}^N \to \mathbb{R}^N$ and $\mathcal{C}_{\Psi(\boldsymbol{\eta})} : \mathbb{R}^N \to \mathbb{R}^m$
\begin{equation}
    \label{eq:rgwvp}
    \begin{split}
        & \mathcal{P}_{\Psi(\boldsymbol{\eta})}(\boldsymbol{f}) := \Psi(\boldsymbol{\eta}) \left( \Psi(\boldsymbol{\eta})^{+} \boldsymbol{f}\right), \\
        & \mathcal{C}_{\Psi(\boldsymbol{\eta})}(\boldsymbol{f}) := \Psi(\boldsymbol{\eta})^{+} \boldsymbol{f},
    \end{split}
\end{equation}
where 
\begin{equation}
\label{eq:eta}
    \boldsymbol{\eta} :=
    \left[\begin{array}{c}
         \lambda_1  \\
         \tau_1  \\
         \vdots \\
         \lambda_m \\
         \tau_m \\
         t_1 \\
         \vdots \\
         t_p \\
         a_1 \\
         b_1 \\
         \vdots \\
         a_n \\
         b_n
    \end{array}\right]
    \in \mathbb{R}^{2m + p + 2\cdot n}, \quad (m,p,n \in \mathbb{N})
\end{equation}
and the columns of $\Psi(\boldsymbol{\eta})$ are samplings of rational Gaussian wavelets defined according to Eq.~\eqref{eq:RGW}. In Eq.~\eqref{eq:eta}, $\lambda_k, \tau_k$ denote the learnable scales and translations, while $t_j$ stand for the zeros of $P^{\boldsymbol{\eta}}$ and $a_i, b_i$ determine the real and imaginary parts of the poles of $v^{\boldsymbol{\eta}}$ in Eq.~\eqref{eq:RGW}. The mapping $\mathcal{C}_{\Psi(\boldsymbol{\eta})}$ in Eq.~\eqref{eq:rgwvp} will be referred to as an "RGW-VP" layer henceforth.

We would like to optimize $\boldsymbol{\eta}$ together with the weights of an underlying neural network. This would ensure that features (wavelet coefficients) of the input signal are learned by the RGW-VP layer which are ideal from the point of view of the underlying classification or regression problem. In addition, this would enable the model to learn the appropriate mother wavelet to perform the transformation with by tuning the parameters $t_j$, $a_i$ and $b_i$ in $\boldsymbol{\eta}$. In~\cite{kovacs2022vpnet}, so-called VP-NETs were proposed, which consist of a variable projection layer (performing adaptive feature extraction), followed by a number of fully connected linear layers (or any other layer type associated with common neural network architectures). The main idea is that since the partial derivatives satisfy  $\frac{\partial \mathcal{C}_{\Psi(\boldsymbol{\eta})}}{\partial \boldsymbol{\eta}} = \frac{\partial \Psi(\boldsymbol{\eta})^+}{ \partial \boldsymbol{\eta} } \boldsymbol{f}$ and $\frac{\partial \Psi(\boldsymbol{\eta})^+}{ \partial \boldsymbol{\eta} }$ can be expressed analytically. The required formulas were first obtained by Pereyra and Golub in~\cite{golub1973differentiation} (see also~\cite{kovacs2022vpnet, vpsvm} for their application in a machine learning context). Thus, it suffices to calculate the partial derivatives $\frac{\partial \Psi(\boldsymbol{\eta})}{\partial \boldsymbol{\eta}}$ in order to supplement a neural network (or other ML model trained by backpropagation) with an RGW-VP layer.

First, we consider the derivatives $\partial \psi^{\boldsymbol{\eta}}_{\lambda_k, \tau_k} / \partial \lambda_k$ and $\partial \psi^{\boldsymbol{\eta}}_{\lambda_k, \tau_k} / \partial \tau_k$. Regardless of the mother wavelet $\psi^{\boldsymbol{\eta}}$, these are given by
\begin{equation}
    \label{eq:deraff}
    \begin{split}
        & \frac{\partial \psi_{\lambda_k, \tau_k}^{\boldsymbol{\eta}}}{\partial \lambda_k} = -\frac{1}{2} \cdot \lambda_k^{-3/2} \cdot \psi_{\lambda_k, \tau_k}^{\boldsymbol{\eta}}(t) - \frac{t - \tau_k}{\lambda_k^2} \cdot \left(\psi_{\lambda_k, \tau_k}^{\boldsymbol{\eta}}(t) \right)' \\
        & \frac{\partial \psi_{\lambda_k, \tau_k}^{\boldsymbol{\eta}}}{\partial \tau_k} = - \left(\psi_{\lambda_k, \tau_k}^{\boldsymbol{\eta}}(t) \right)'/\lambda_k,
    \end{split}
\end{equation}
where $\left(\psi^{\boldsymbol{\eta}}_{\lambda_k, \tau_k}(t) \right)'$ denotes the derivative of $\psi_{\lambda_k, \tau_k}^{\boldsymbol{\eta}}$ with respect to $t$. 

In order to give the required derivatives with respect to the zeros of $P^{\boldsymbol{\eta}}$ and the poles of $v^{\boldsymbol{\eta}}$, first consider that the elementary polynomials from Eq.~\eqref{eq:baspols} may be written as
\begin{equation*}
    r_{z}(t) = t^4 + t^2\cdot 2(b^2 - a^2) + a^4 + 2 a^2 b^2 + b^4,
\end{equation*}
where $a := \Re(z)$ and $\b := \Im(z)$. Recall, that in section~\ref{sec:rgw}, we required that $b > 0$. One way to ensure this, is to introduce $\hat{b} := b^2 + \varepsilon$, for some $\varepsilon > 0$ and modify the formulas presented here accordingly. Thus, using the notation $z_k := a_k + i b_k$, we obtain
\begin{equation*}
\begin{split}
    & \frac{\partial r_{z_k}(t)}{ \partial a_k} = -4 \cdot \left( a_k t^2 - a_k^3 - a_k b_k^2 \right) \\
    & \frac{\partial r_{z_k}(t)}{ \partial b_k} = 4 \cdot \left( b_k t^2 + b_k^3 + a_k^2 b_k \right)
\end{split}    
\end{equation*}
Notice that the $v^{\boldsymbol{\eta}}$ can be written as the finite product of $r_{z_k}(t)^{-1}$ terms. The partial derivatives of these terms are given by 
\begin{equation*}
    \begin{split}
        & \frac{\partial r_{z_k}(t)^{-1}}{\partial a_k} = -r_{z_k}(t)^{-2} \cdot \frac{\partial r_{z_k}(t)}{\partial a_k} \\
        & \frac{\partial r_{z_k}(t)^{-1}}{\partial b_k} = -r_{z_k}(t)^{-2} \cdot \frac{\partial r_{z_k}(t)}{\partial b_k}.
    \end{split}
\end{equation*}
Using this, we obtain the following partial derivatives for $v^{\boldsymbol{\eta}}$ with respect to $a_k$ and $b_k$:
\begin{equation}
    \label{eq:partv}
    \begin{split}
        & \frac{\partial v^{\boldsymbol{\eta}}(t)}{\partial a_k} = \frac{\partial r_{z_k}(t)^{-1}}{\partial a_k} \cdot \prod_{j=0, j \neq k }^{n-1} r_{z_j}(t) \\
        & \frac{\partial v^{\boldsymbol{\eta}}(t)}{\partial b_k} = \frac{\partial r_{z_k}(t)^{-1}}{\partial b_k} \cdot \prod_{j=0, j \neq k }^{n-1} r_{z_j}(t).
    \end{split}
\end{equation}
Finally, using Eq.~\eqref{eq:partv} the formula for the partial derivatives of the mother wavelet can be given as
\begin{equation}
    \label{eq:dpsiab}
    \begin{split}
        & \frac{\partial \psi_{\lambda_i, \tau_i}^{\boldsymbol{\eta}}(t)}{\partial a_k} = P^{\boldsymbol{\eta}}(t) \frac{\partial v^{\boldsymbol{\eta}}(t)}{\partial a_k} e^{-t^2/2} \\
        & \frac{\partial \psi_{\lambda_i, \tau_i}^{\boldsymbol{\eta}}(t)}{\partial a_k} = P^{\boldsymbol{\eta}}(t) \frac{\partial v^{\boldsymbol{\eta}}(t)}{\partial b_k} e^{-t^2/2}.
    \end{split}
\end{equation}
Note that in Eq.~\eqref{eq:dpsiab}, for the sake of simplicity, we abstained from adding the effect of the normalizing factor $C(\boldsymbol{\eta})$. The partial derivatives also considering this factor can be easily calculated using the theorem of parametric integration and the above equations, however we found (see section~\ref{sec:expr}), that in applications good results could be achieved without it. The partial derivatives of $\psi^{\boldsymbol{\eta}}_{\lambda_i, \tau_i}$ with respect to the zeros of $P^{\boldsymbol{\eta}}$ can also be easily calculated:
\begin{equation}
    \label{eq:dpsitk}
    \frac{\partial \psi_{\lambda_i, \tau_i}^{\boldsymbol{\eta}}}{\partial t_k}(t) = - \prod_{j=1, j\neq k}^{p}(t-t_j)(t+t_j) \cdot 2 t_k \cdot v^{\boldsymbol{\eta}}(t) \cdot e^{-t^2/2}.
\end{equation}
Equations~\eqref{eq:dpsiab} and~\eqref{eq:dpsitk} can be used to implement the backpropagation step for the RGW-VP layer. Fig.~\ref{fig:rgwvpnet} shows a schematic of the proposed neural network arhitecture using the RGW-VP layer for automatic feature extraction. The main benefit of approximating wavelet coefficients using variable projection is that for any fixed $\boldsymbol{\eta}$ parameter vector, the computed coefficients minimize the error of the reconstructed signal in the least-squares sense. In order to ensure high quality approximations of the input signals and to avoid the problem of vanishing gradients, in VP-based ML methods the following regularization term is added to the loss function during training (see~\cite{kovacs2022vpnet, eusipco}):
\begin{equation*}
    J_{VP}(\boldsymbol{\eta}) := \frac{\alpha}{s} \sum_{i=1}^s \frac{\| \boldsymbol{f}_i - \mathcal{P}_{\Psi(\boldsymbol{\eta})}\boldsymbol{f}_i \|_2^2}{\| \boldsymbol{f}_i \|_2^2}.
\end{equation*}
Here $\mathcal{P}_{\Psi(\boldsymbol{\eta})}\boldsymbol{f}$ is defined according to Eq.~\eqref{eq:rgwvp}, $\alpha > 0$ is a hyper-parameter and $s$ denotes the total number of input signals in the training set.
\begin{figure*}[h!]
    \centering
    \includegraphics[width=0.8\linewidth]{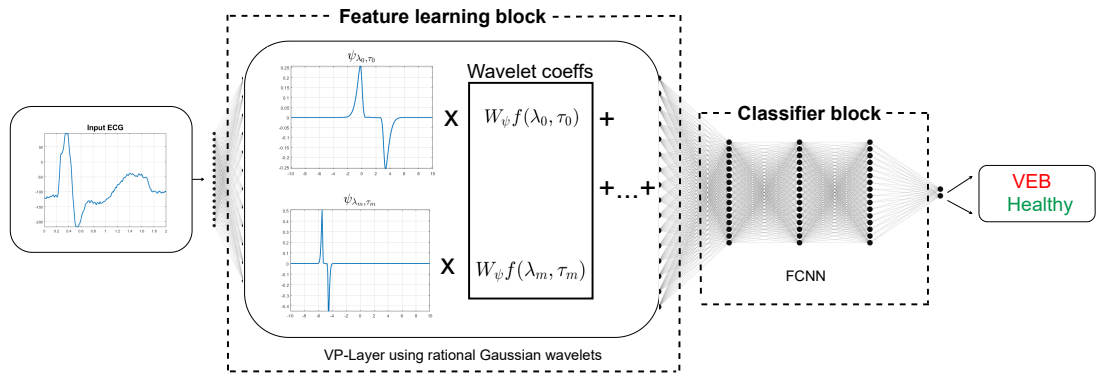}
    \caption{Schematic of the proposed, rational Gaussian wavelet based neural network architecture. The first layer acts as an automatic feature learning block passing the learned wavelet coefficients to an underlying fully connected network.}
    \label{fig:rgwvpnet}
\end{figure*}

\section{A case study: ECG classification}
\label{sec:expr}

Electrocardiogram (ECG) is an effective and low-cost method for monitoring heartbeats. Detecting heart arrhythmia can sometimes be challenging due to the infrequent appearance of certain types of arrhythmias, which require long-term diagnostic periods. To aid cardiologists and people living in remote areas, the automatic classification of heartbeat signals has been researched for decades \cite{de2004automatic}. A healthy ECG heartbeat consists of three segments: the P wave, the QRS complex, and the T wave, respectively. The P wave corresponds to atrial depolarization, the QRS complex represents the depolarization of the ventricles, and the T wave represents the repolarization of the ventricles \cite{thaler2021only}. Cardiologists examine the properties of these segments, such as interval length, wave shape, or wave presence, to classify arrhythmias \cite{thaler2021only}. For simplicity of notation, any numerical approximation of the integral~\eqref{eq:wavcoeff} will be referred to as "CWT coefficients" henceforth, including the variable projection based approximation discussed in the previous section.
\begin{figure}
    \centering
    \includegraphics[width=1.0\linewidth]{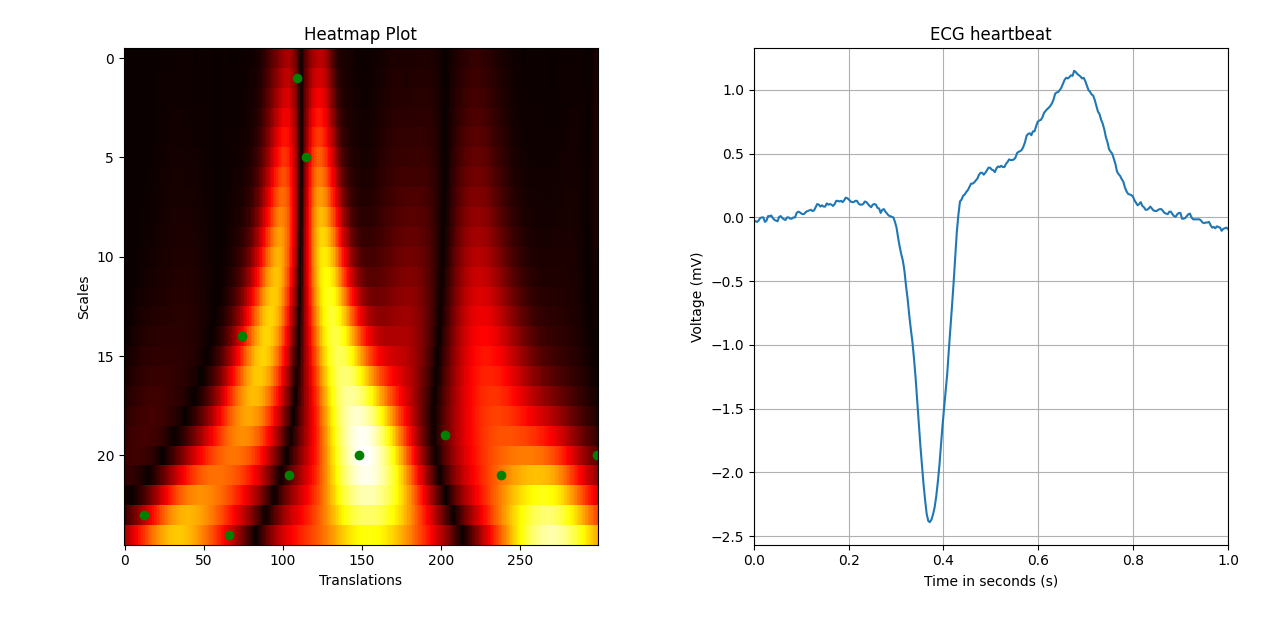}
    \caption{CWT coefficients of a VEB signal represented by their scalogram and a VEB signal plot, respectively. Trained coefficients marked with green dots on the scalogram.}
    \label{fig:scalogram}
\end{figure}

There are other methods in the literature, which used CWT coefficients to classify ECG heartbeats. In most approaches so-called scalograms are used. These can be given as the matrices $W \in \mathbb{R}^{r \times k}$, where $W_{r,k} := | W_{\psi} f (\lambda_r, \tau_k) | \ (r,k \in \mathbb{N})$. For example, in \cite{al2018convolutional}, a static feature extraction scheme was applied on the input heartbeats. Namely, the scalograms with respect to three different mother wavelets were computed and represented as a tensor of size $3 \times r \times k$. Clearly, such a tensor can be thought of as an RGB image. This idea was exploited to pass the preprocessed signals (tensors) to a deep convolutional (CNN) network architecture, which expects RGB images as inputs. In \cite{ fujita2019decision}, the Morse analytic mother wavelet was used to compute scalograms of groups of successive ECG heartbeats, then they were fed into a deep CNN model. Another approach used scalograms of heartbeats and interval features as input to a CNN in order to classify arrhythmias \cite{wang2021automatic}. In this article, it was found that the Ricker wavelet was the most suitable for this task. The authors of~\cite{wang2021automatic} claim that this is due to the shape similarities between QRS complexes and the Ricker wavelets. 

One common property of these methods is that they compute thousands of CWT coefficients providing a redundant signal representation, and in turn require deep CNNs to solve the classification task. CNNs extract features automatically and show state-of-the-art performance for various image or signal processing tasks. As we mentioned earlier, one disadvantage of such deep learning models is that they use a large number of parameters and thus are difficult to interpret. CWT-VP layers achieve sparse signal representation by approximating only the  CWT coefficients necessary for the classification task. Sparse signal representation also allows lower model complexity. Additionally, CWT-VP layers do not require pooling or dropout layers to reduce input dimensions.  

In \cite{kovacs2022vpnet} VP-Net was used with Adaptive Hermite functions to classify ECG heartbeats. Adaptive Hermite functions were utilized for ECG processing in~\cite{dozsa2016ecg}, although various other parameterizations of Hermite functions have long been used for approximating these kinds of signals. Hermite-functions form a complete and orthonormal system in $L_2(\mathbb{R})$ and show similar morphology to healthy QRS complexes. For this reason, the first few Hermite-Fourier coefficients may be used to provide a good approximation of QRS complexes. Adaptive Hermite-functions consider the affine argument transforms of the classical Hermite functions. In this case, the VP-Layer learns a single dilation and translation parameter pair, which describes the entire function system.

In our experiments, we classified ventricular ectopic beat (VEBs) signals, which are the most common heartbeat irregularities \cite{de2004automatic}. We used heartbeats from the MIT-BIH arrhythmia database \cite{moody2001impact}, provided by PhysioNet \cite{goldberger2000physiobank}. In this database, a total of 48 recordings from different patients can be found. We used the proposed datasets from~\cite{de2004automatic} by De Chazal et al. for training and testing respectively. These datasets are unbalanced, since in real life scenarios, normal heartbeats are more common, than abnormal ones. The datasets contains approximately 50,000 heartbeats. We represent every heartbeat with 300 data samples, which includes the P segment, QRS complex and T segment. It is important to note, that many current approaches only consider the QRS complex (see~\cite{luz2016ecg}), however as we point out in section~\ref{subsec:interp}, biological markers such as the shape of the P-wave are also used by cardiologists to recognize VEBs.

In our experiment we considered the proposed RGW-VP-Net architecture using $m=10$ wavelet coefficients, $p=3$ zeros and $n=4$ poles. We compared the performance of the proposed model with a VP-Net using the adaptive Hermite system as introduced in~\cite{kovacs2022vpnet}. We note that the Hermite expansion based VP-Net also received the entire heartbeat as input and used $m=10$ Hermite-functions. In addition we considered a variable projection based wavelet network, where the VP-wavelet transformation from Eq.~\eqref{eq:rgwvp} was performed using the Ricker wavelet. This model architecture was included to measure the effect of the proposed RGW wavelets, where the shape of the mother wavelet can also be learned. In the Ricker-wavelet case, the learnable parameters of the VP-layer consisted of the dilation-translation parameter pairs associated with each wavelet coefficient. Again, $m=10$ coefficients were considered. Finally, we considered the performance of well-known methods from~\cite{luz2016ecg}. The classification results are shown in Table \ref{tab:experiments_ecg}. In addition to overall accuracy, common performance metrics, sensitivity/precision (Se)  and Positive predictivity/recall (+P), also calculated:
\begin{equation}
    Se = \frac{TP}{TP + FN} \quad \text{and} \quad {+P} = \frac{TP}{TP+FP}
\end{equation}
where TP, FN, FP are true positive, false negative and  false positive, respectively. These metrics were considered, since the dataset was highly biased towards healthy signals.

\begin{table}[!h]
    \centering
    \begin{tabular}{cccccc}
      	\toprule
        \multirow{ 2}{*}{Layer } & \multirow{ 2}{4em}{Total accuracy } & \multicolumn{2}{c}{Normal} &  \multicolumn{2}{c}{VEB}  \\ &  & Se & +P & Se & +P\\
		\midrule
   {\bf RGW-VP} & {\bf 98.51\%} & {\bf 99.60\%} & {\bf 98.79\%}  & {\bf 85.48\%} & {\bf 94.71\%} \\
  {\bf Ricker-VP}  & 95.69\%  & 99.37\% & 96.1\%   & 51.68\% &  87.27\%\\
		  {\bf Hermite-VP}  &  97.42\% & 98.65\%   & 98.56\% &  82.75\% & 83.63\% \\ \midrule
    {\bf Literature~\cite{luz2016ecg}} & - & 80-99 \% & 85-99 \% & 77-96 \% & 63-99 \% \\
		\bottomrule
    \end{tabular}
    \vspace{0.1cm}
    \caption{VEB classification results on the MIT-BiH Arrythmia dataset.}
    \label{tab:experiments_ecg}
\end{table}

Based on the results in table~\ref{tab:experiments_ecg}, the proposed RGW-VP model outperformed the other considered approaches across all metrics. It should be noted, that it also perfomred at the level of the best methods from the literature~\cite{luz2016ecg}, some of which used much more complex, deep learning models for classification. From the medical point of view, the VEB Se and VEB +P metrics are particularly important, as they provide information about how reliably the methods recognize diseased waveforms. Unfortunately, in this case the Ricker-VP found just half of the VEB signals. The Hermite function based VP-Net achieved better results than Ricker-VP and comparable results to RGW-VP, but RGW-VP outperforms Hermite-VP in the VEB +P metric by approximately 10\%. It should also be emphasized, (see also section~\ref{subsec:interp}), that the learned parameters of RGW-VP are highly interpretable when compared to non-wavelet based approaches.

In order to measure the effect of different feature learning layers, all models used the same underlying fully-connected architecture. The hyperparameters of the model tuned using a large, manual grid search. Sparse signal representations achieved by the VP automatic feature extraction layers reach peak performance at low model complexity \cite{kovacs2022vpnet}. In the first layer we computed $m=10$ coefficients in all cases.  We tested Ricker and Hermite-VP with more coefficients, but the performance did not outperform RGW-VP. Each model used a single hidden linear layer with 15 neurons and with a ReLu activation function after the first layer. The one-dimensional output layer used the Sigmoid function for binary classification. 

\subsection{Interpretability of proposed models}
\label{subsec:interp}

In this section we examine the interpretability of the trained parameters of the first layer. According to the literature, cardiologist recognise VEBs from their wide and bizarre QRS complex morphologies and the absence of a P wave in the heartbeat \cite{thaler2021only}. This observation supports the use of the entire heartbeat as input to the models, unlike previous methods (see e.g.~\cite{de2004automatic, kovacs2022vpnet}), which only considered QRS complexes. The right column of Fig.~\ref{fig:scalogram} illustrates a VEB haertbeat. In the left column of the same figure, we show the scalogram of the heartbeat computed using the RGW mother wavelet learned in our experiment. The scalogram was acquired by computing all wavelet coefficients with the RGW-VP layer at every time point and 25 scales between  $\lambda_1=0.1$ and $\lambda_{25}=1.5$. We can see that there are dominant coefficients at lower frequencies (higher scales) because of the wide QRS complex of the heartbeat. If we examine the positions of the trained coefficients - marked with green dots - we conclude that they correspond to lower frequency information of VEB signals. At the QRS complex time segment there are also trained coefficients at high frequency values (i.e., at low scales) to aid the classification of healthy ECG heartbeat signals. Fig.~\ref{fig:wavelets} illustrates the learned dilated and translated RGW wavelets used to extract the features from the ECG data. The figure shows that the wavelets correspond to segments which cardiologist examine to classify VEB signals. The dark blue RGW wavelet corresponds to the P wave, light green wavelet corresponds to a wide QRS complex of a VEB signal, orange and light blue wavelets correspond to a QRS complex of a healthy ECG heartbeat. In Figure \ref{fig:coeffs} we can see a healthy and an unhealthy ECG heartbeat. In the bottom row, the value of the wavelet coefficients is also illustrated. In the case of a healthy ECG heartbeat, the coefficients which correspond to the P wave (in time) were relatively large. In turn, a VEB signal has a small dark blue coefficient value, because of the absence of the P wave. In the case of VEB signals, we can clearly see that the light green coefficient is bigger than coefficients corresponding to healthy QRS complexes.
\begin{figure}
    \centering
    \includegraphics[width=1\linewidth]{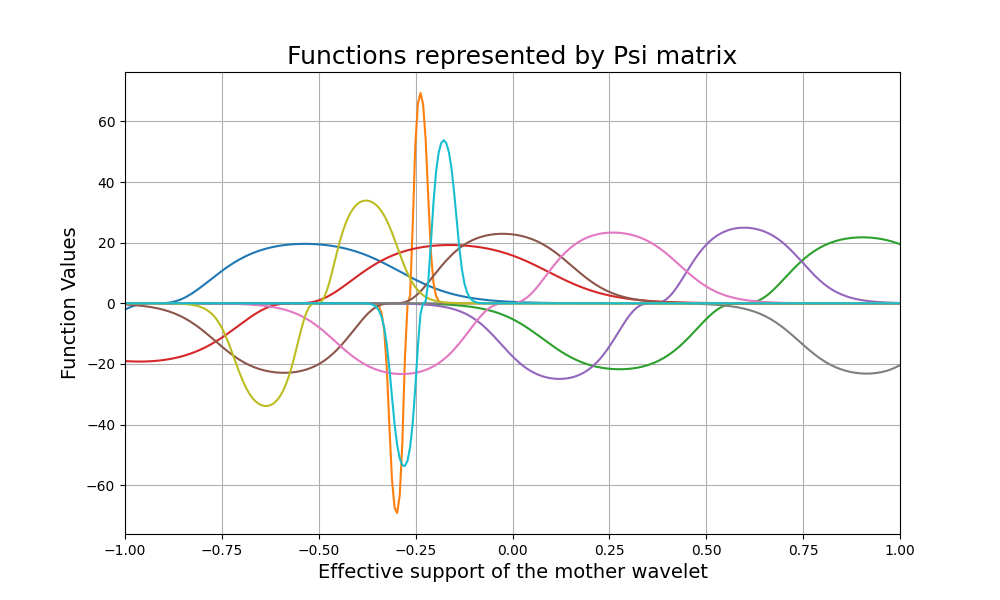}
    \caption{Dilated and translated RGW wavelets used to extract features from ECG signals.}
    \label{fig:wavelets}
\end{figure}

\begin{figure}
    \centering
    \includegraphics[width=1\linewidth]{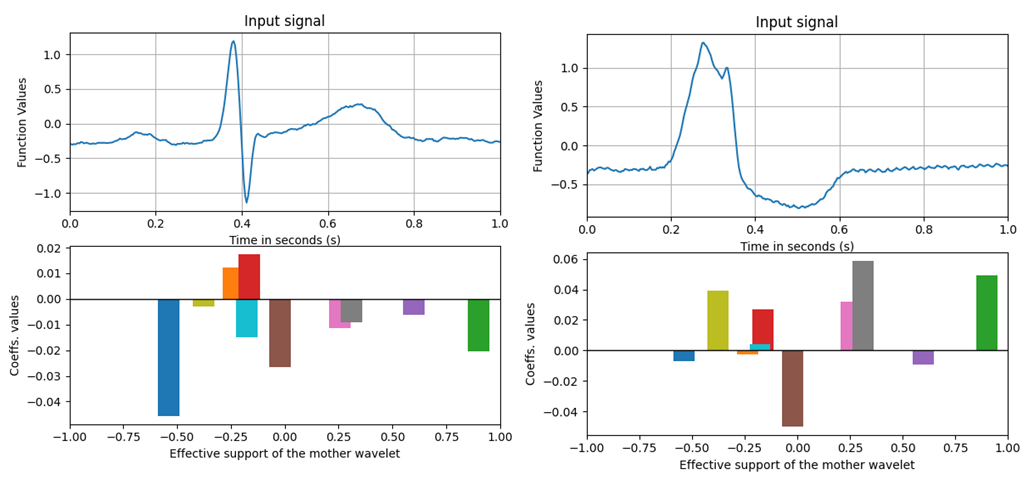}
    \caption{A healthy and VEB ECG heartbeat signal. Below there are the values of their CWT coefficients plotted at maximum function values of their corresponding wavelets. Colors corresponds to the wavelets in Fig.~\ref{fig:wavelets}.}
    \label{fig:coeffs}
\end{figure}

\section{Conclusion}
\label{sec:conc}

In this work, we proposed a new analytic family of wavelets which we referred to as rational Gaussian wavelets. The main idea behind the RGW wavelets is to multiply the Gaussian with an appropriate rational term, whose poles and zeros act as free parameters that govern the morphology of the mother wavelet. We demonstrated that using RGW wavelet coefficients, non-symmetric and sharp waveforms (such as ECG signals) can be reconstructed using only a few wavelet parameters, provided that the mother wavelet morphology is chosen appropriately. This allows for a precise and sparse representation of a large class of signals. Furthermore, we proved the admissibility of the proposed RGW wavelets and considered an appropriate numerical approximation of wavelet coefficients which minimize the error of the reconstructed signal. We showed that this method of approximating RGW coefficients using variable projection can be effectively included in neural network architectures to solve a real life signal processing problem, namely VEB detection in ECG signals. Our results show that the proposed methodology performs this task at the level of well-established alternatives from the literature, while providing interpretable and simple model architectures. Importantly, the proposed RGW-VP neural network was shown to learn wavelet coefficients, whose time parameter corresponds to ECG segments that are crucial in the detection of VEBs according to cardiology literature. In this way, the proposed transformation provides highly interpretable features while obtaining a high quality sparse representation of ECG data. Due to these advantages, RGW-VP was able to outperform other model-driven architectures, such as wavelet based neural networks using a fixed analyzing wavelet and variable projection networks using an adaptive Hermite expansion of the signals.

The results disclosed in this work demonstrate the potential in the proposed RGW-VP architecture. In our future work, we would like to find other potential applications of the RGW-VP scheme. In addition, we would like to investigate the possibility of using RGW wavelets for other tasks as well, such as feature extractors for industrial fault detection problems. In addition, we would like to investigate whether this construction may be useful in other settings entirely, such as univariate function layers of Kolmogorov-Arnold Networks.

\appendices
\section{Proofs of statements~\ref{thm:admisrgw} and~\ref{them:errorest}}
\label{sec:proofs}

\begin{thm}[Admissibility of rational Gaussian wavelets (Theorem~\ref{thm:admisrgw})]
    Suppose $\psi \in L_2(\mathbb{R})$ is defined according to Eq.~\eqref{eq:RGW}. Then, there exists a number $M \in \mathbb{R}_+$, such that
    \begin{equation*}
        \int_{-\infty}^{\infty} \frac{1}{|\xi|} \left| \widehat{\psi}(\xi) \right|^2 d\xi = M
    \end{equation*}
    holds.
\end{thm}
\begin{proof}
Recall that rational Gaussian wavelets were defined according to Eq.~\eqref{eq:RGW}, that is
\begin{equation*}
    \Psi^{\boldsymbol{\eta}}(t) := C(\boldsymbol{\eta}) \cdot P^{\boldsymbol{\eta}}(t) \cdot v^{\boldsymbol{\eta}}(t) \cdot e^{-t^2/2},
\end{equation*}
where $P^{\boldsymbol{\eta}}$ is an odd polynomial symmetric to $0$ and $v^{\boldsymbol{\eta}}$ belongs to the class of rational functions given in Eq.~\eqref{eq:ratclass}. As explained in section~\ref{sec:rgw}, the parameter vector $\boldsymbol{\eta}$ fully defines the zeros and poles of $P^{\boldsymbol{\eta}} \cdot v^{\boldsymbol{\eta}}$. From Eq.~\eqref{eq:ratclass}, it is clear that $\lim_{t \to \pm \infty} v^{\boldsymbol{\eta}}(t) = 0$, and since $v^{\boldsymbol{\eta}}$ has no real poles we have
\begin{equation*}
    \left| v^{\boldsymbol{\eta}}(t) \right| < V < \infty \quad (t \in \mathbb{R}).
\end{equation*}
The Gaussian term $e^{-t^2/2}$ belongs to the class of Schwartz-functions denoted by $\mathcal{S}(\mathbb{R})$. It is well-known that $\mathcal{S}(\mathbb{R}) \subset L_1(\mathbb{R})$ and if $f \in \mathcal{S}(\mathbb{R})$, then $Q \cdot f \in \mathcal{S}(\mathbb{R})$, where $Q$ is an arbitrary polynomial. From these observations we obtain
\begin{equation}
    \label{eq:l1xpsi}
    \begin{split}
    & \int_{-\infty}^{+\infty} \left| t \cdot \psi^{\boldsymbol{\eta}}(t) \right| dt = C(\boldsymbol{\eta}) \int_{-\infty}^{+\infty} \left| t \cdot P^{\boldsymbol{\eta}}(t) v^{\boldsymbol{\eta}}(t) e^{-t^2/2} \right| dt = \\ & C(\boldsymbol{\eta}) \int_{-\infty}^{+\infty} \left| t \cdot P^{\boldsymbol{\eta}}(t) \cdot e^{-t^2/2    } \right| \left| v^{\boldsymbol{\eta}}(t) \right| dt \\ & < C(\boldsymbol{\eta}) \cdot V \cdot \int_{-\infty}^{+\infty} \left| Q^{\boldsymbol{\eta}}(t) \cdot e^{-t^2/2} \right| dt < \infty,
    \end{split}
\end{equation}
since $Q^{\boldsymbol{\eta}}(t)$ is a polynomial, we have $t \cdot P^{\boldsymbol{\eta}}(t) \cdot e^{-t^2/2} \in \mathcal{S}(\mathbb{R}) \subset L_1(\mathbb{R})$. In other words we find that $t \cdot \psi^{\boldsymbol{\eta}}$ is in $L_1(\mathbb{R})$.

Notice, that $\psi^{\boldsymbol{\eta}}$ is an odd function, that is $\psi(t) = -\psi(-t) \ (t \in \mathbb{R})$. This is so, because $e^{-t^2/2}$ and $v^{\boldsymbol{\eta}}$ are even and $P^{\boldsymbol{\eta}}$ is an odd polynomial satisfying Eq.~\eqref{eq:polpart}. Using the oddity of $\psi^{\boldsymbol{\eta}}$ we obtain
\begin{equation}
    \label{eq:psifour0}
    \widehat{\psi^{\boldsymbol{\eta}}}(0) = \int_{-\infty}^{+ \infty} \psi^{\boldsymbol{\eta}}(t) dt = \int_{0}^{+ \infty} \psi^{\boldsymbol{\eta}}(t) + \psi^{\boldsymbol{\eta}}(-t) dt = 0.
\end{equation}
Our claim is a simple consequence of Eqs.~\eqref{eq:l1xpsi} and~\eqref{eq:psifour0}. Since $\psi^{\boldsymbol{\eta}}$ and $t \cdot \psi^{\boldsymbol{\eta}}$ are both integrable, by the theorem regarding the differentiability of the Fourier transform (see e.g. Proposition 17.2.1 in~\cite{gasquet2013fourier}) we have $\widehat{\psi^{\boldsymbol{\eta}}} \in C^{1}(\mathbb{R})$. However if this is the case, then the function $\widehat{\psi^{\boldsymbol{\eta}}}(\xi) / \xi$ is continuous at $0$. This is true, since in $0$ by definition, the derivative of $\widehat{\psi^{\boldsymbol{\eta}}}$ is expressed by
\begin{equation*}
\left(\widehat{\psi^{\boldsymbol{\eta}}}\right)'(0) := \lim_{\xi \to 0} \frac{\widehat{\psi^{\boldsymbol{\eta}}}(\xi) - \widehat{\psi^{\boldsymbol{\eta}}}(0)}{\xi - 0} = \lim_{\xi \to 0} \frac{\widehat{\psi^{\boldsymbol{\eta}}}(\xi)}{ \xi},
\end{equation*}
where we used Eq.~\eqref{eq:psifour0} for the second equality. If $\xi \to \infty$, then the integral in theorem~\ref{thm:admisrgw} certainly exists, since $\psi^{\boldsymbol{\eta}} \in L_2(\mathbb{R})$. This concludes our proof of statement~\ref{thm:admisrgw}.
\end{proof}

\begin{thm}[Error of variable projection based continuous wavelet coefficients (Theorem~\ref{them:errorest})]
    Let $\psi \in L_2(\mathbb{R})$ be a fixed mother wavelet and suppose $f \in L_2(\mathbb{R})$ is continuously differentiable and compactly supported on $[a, b] \subset \mathbb{R}$. Let $a =: t_0 < t_1 < \ldots < t_{N-1} := b$ be an equidistant sampling of $[a, b]$ with $h := t_{1} - t_0$. In addition, consider $\boldsymbol{\eta} = (\lambda_1, \tau_1, \ldots, \lambda_m, \tau_m) \in \mathbb{R}^{2m}$ with $\lambda_k > 0 \ (k=1,\ldots,m)$. Finally, let $\boldsymbol{f}_k := f(t_k) \ (k=0,\ldots,N-1)$. Then we have
    \begin{equation*}
        \begin{split}
        & \left|W_{\psi}f(\lambda_k, \tau_k) - h \cdot \left(\Psi(\boldsymbol{\eta})^+ \boldsymbol{f} \right)_k \right| < \\
        & h \cdot \frac{M_1 (b - a)}{2} + h \cdot\|\boldsymbol{f}\|_{\infty} \|\Psi^{*}(\boldsymbol{\eta}) \|_{\infty} \cdot \\ &  \left( \frac{\kappa(G(\boldsymbol{\eta})) }{\|G(\boldsymbol{\eta})\|_{\infty}} + 1 \right),
        \end{split}
    \end{equation*}
    where 
    \begin{equation*}
        M_1 := \max_{\xi \in [a, b], k=1,\ldots,m} \left|f'(\xi) \cdot \overline{\psi_{\lambda_k, \tau_k}}(\xi) \right|,
    \end{equation*}
    the columns of $\Psi(\boldsymbol{\eta})$ contain the samplings of $\psi_{\lambda_k, \tau_k} \ (k=1,\ldots,m)$ over $t_j \ (j=0,\ldots,N-1)$ and $G(\boldsymbol{\eta})$ is the Gram-matrix constructed form the columns of $\Psi(\boldsymbol{\eta})$. In Eq.~\eqref{eq:errest}, $\kappa(G(\boldsymbol{\eta}))$ denotes the condition number of $G(\boldsymbol{\eta})$ using the matrix infinity norm and $\Psi^{*}(\boldsymbol{\eta})$ denotes the adjungate of $\Psi(\boldsymbol{\eta})$.    
\end{thm}

\begin{proof}
    Since the support of $f$ coincides with $[a, b]$, we have
    \begin{equation*}
        W_{\psi}(\lambda, \tau)f = \int_{a}^{b} f(t) \overline{\psi_{\lambda, \tau}}(t) dt
    \end{equation*}
    for all $(\lambda, \tau) \in \mathbb{R}_{+} \times \mathbb{R}$. This integral can be approximated using the numerical quadrature
    \begin{equation*}
        E(f, \lambda, \tau) := \sum_{j=0}^{N-1} f(t_j) \cdot \overline{\psi_{\lambda, \tau}}(t_j)
    \end{equation*}
    with the well-known error estimate (following from the mean value theorem regarding integration)
    \begin{equation}
        \label{eq:quadest}
        | W_{\psi}(\lambda, \tau)f - h \cdot E(f, \lambda, \tau) | < h \cdot \frac{M_1 (b-a)}{2}.
    \end{equation}
    Let the parameter pairs $(\lambda_k, \tau_k) \ (k=1,\ldots,m)$ be fixed and notice, that for a fixed $k$, the value of $E(f)$ can be expressed by
    \begin{equation*}
        E(f, \lambda_k, \tau_k) = \left( \Psi^{*} \boldsymbol{f} \right)_k,
    \end{equation*}
    where $\boldsymbol{f}$ denotes the sampling of $f$ over the points $t_j \ (j=0,\ldots,N-1)$ and $\Psi^{*}$ denotes the adjungate of $\Psi := \Psi(\boldsymbol{\eta})$. From this it is clear that the estimate in Eq.~\eqref{eq:quadest} only holds for the approximations obtained using variable projection, if $\Psi$ is unitary, i.e. $\Psi^{+} = \Psi^{*}$. If the columns of $\Psi$ contain analytic wavelets (for example RGW), this cannot be assumed, therefore it is necessary to estimate
    \begin{equation*}
        | \left( \Psi^{*} \boldsymbol{f} \right)_k - \left( \Psi^{+} \boldsymbol{f} \right)_k | \quad (k=1,\ldots,m).
    \end{equation*}
    Notice that $\Psi^{+} \boldsymbol{f}$ coincides with the solution to the Gaussian normal equation
    \begin{equation*}
        \Psi^{*} \Psi \boldsymbol{c} = \Psi^{*} \boldsymbol{f}.
    \end{equation*}
    Here we exploited the fact that for pairwise different $\lambda_k, \tau_k$ parameters, $\Psi$ will be full-rank. The solution is expressed by
    \begin{equation*}
        \boldsymbol{c} = \left( \Psi^{*} \Psi \right)^{-1} \Psi^{*} \boldsymbol{f} = G^{-1} \Psi^{*} \boldsymbol{f},
    \end{equation*}
    where $G(\boldsymbol{\eta}) = G = \Psi^{*} \Psi$. Using this, the triangle inequality and the consistency of the $\|\cdot\|_{\infty}$ matrix and vector norms, we obtain
    \begin{equation}
        \label{eq:finest}
        \begin{split}
         &   \| \Psi^{*} \boldsymbol{f} - G^{-1} \Psi^{*} \boldsymbol{f} \|_{\infty} < \|\Psi^{*} \boldsymbol{f} \|_{\infty} + \| G^{-1} \Psi^{*} \boldsymbol{f} \|_{\infty} < \\ & \|\Psi^{*}\|_{\infty} \|\boldsymbol{f} \|_{\infty} \cdot \left( \|G^{-1}\|_{\infty} + 1\right) = \|\Psi^{*}\|_{\infty} \|\boldsymbol{f} \|_{\infty} \left( \frac{\kappa(G)}{\|G\|_{\infty}} + 1\right).
        \end{split}
    \end{equation}
    We once more use the triangle rule together with Eqs.~\eqref{eq:quadest} and~\eqref{eq:finest}:
    \begin{equation*}
        \begin{split}
            & |W_{\psi}f(\lambda_k, \tau_k) - h (\Psi^{+} \boldsymbol{f})_k| = \\ & |W_{\psi}f(\lambda_k, \tau_k) - h (\Psi^{*} \boldsymbol{f} + \Psi^{+} \boldsymbol{f} - \Psi^{*} \boldsymbol{f})_k | = \\
            & |W_{\psi}f(\lambda_k, \tau_k) - h (\Psi^{*} \boldsymbol{f})_k - h (\Psi^{+} \boldsymbol{f} - \Psi^{*} \boldsymbol{f})_k | < \\
            & |W_{\psi}f(\lambda_k, \tau_k) - h (\Psi^{*} \boldsymbol{f})_k| + h \cdot |(\Psi^{+} \boldsymbol{f} - \Psi^{*} \boldsymbol{f})_k| <  \\
            & h \cdot \frac{M_1 (b-a)}{2} + h \cdot \|\Psi^{+} \boldsymbol{f} - \Psi^{*} \boldsymbol{f}\|_{\infty} < \\
            & h \cdot \frac{M_1 (b-a)}{2} + h \cdot \|\Psi^{*}\|_{\infty} \|\boldsymbol{f} \|_{\infty} \left( \frac{\kappa(G)}{\|G\|_{\infty}} + 1\right).
        \end{split}
    \end{equation*}
    This concludes our proof.
\end{proof}

\section*{Acknowledgment}
Project no.\ C1748701 has been implemented with the support provided by the Ministry of Culture and Innovation of Hungary from the National Research, Development and Innovation Fund, financed under the NVKDP-2021 funding scheme. Project no. TKP2021-NVA-29 and K146721 have been implemented with the support provided by the Ministry of Culture and Innovation of Hungary from the National Research, Development and Innovation Fund, financed under the TKP2021-NVA and the K\_23 "OTKA" funding schemes. Supported by the EKÖP-KDP-24 university excellence scholarship program
cooperative doctoral program of the Ministry for Culture and Innovation from the source of the National Research, Development and Innovation Fund.

\ifCLASSOPTIONcaptionsoff
  \newpage
\fi



%
\bibliographystyle{IEEEtran}
\bibliography{refs.bib}

%
\vspace{-1cm}
\begin{IEEEbiography}[{\includegraphics[width=1in,height=1.25in,clip,keepaspectratio]{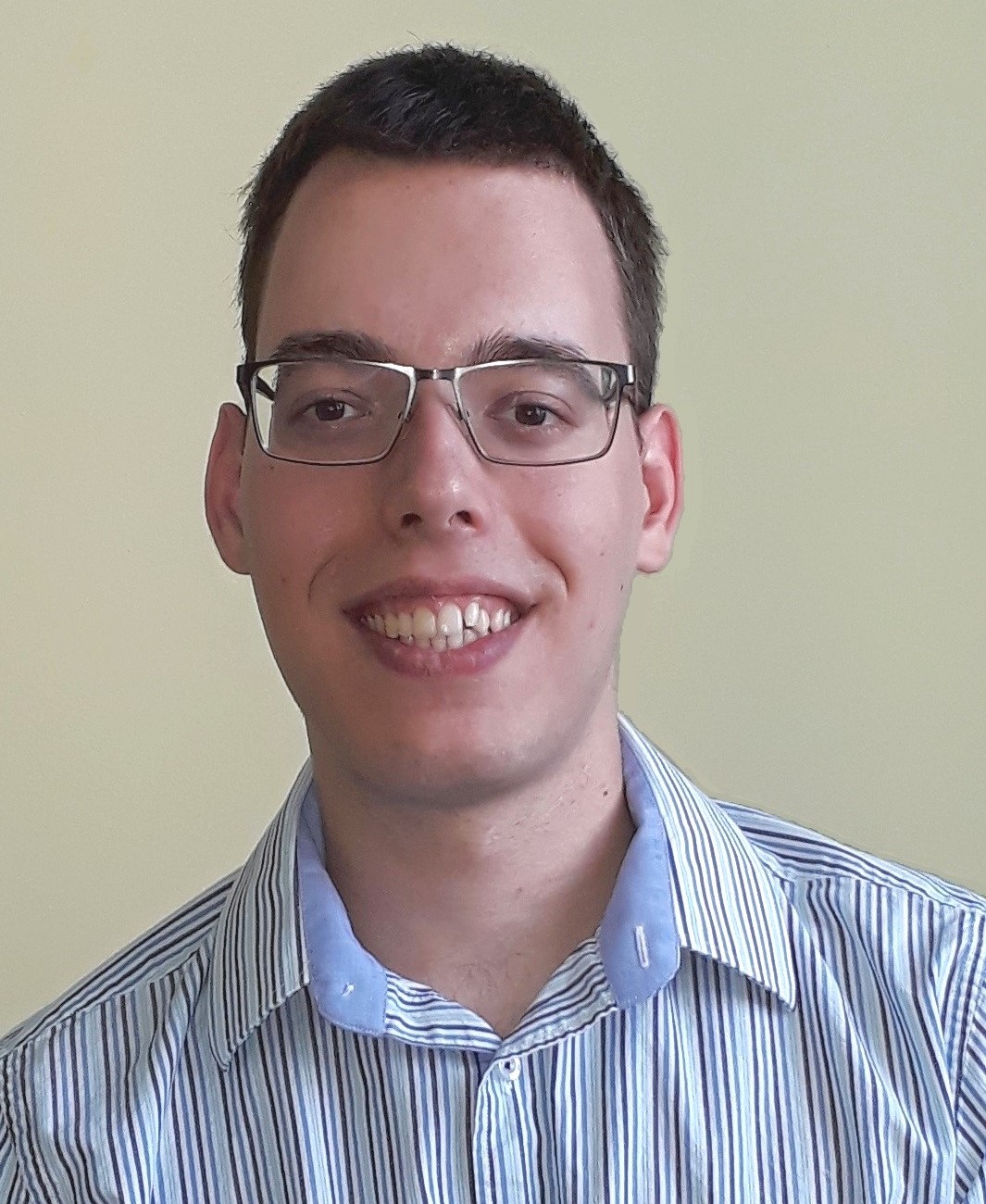}}]{Attila Mikl\'os \'Amon}
received the M.Sc. degree in computer
science from Eötvös Loránd University (ELTE),
Budapest, Hungary, in 2024, where he is currently
pursuing the Ph.D. degree in computer science. 
His research interests
include explainable AI, signal and image processing.
\end{IEEEbiography}
\vspace{-1cm}
\begin{IEEEbiography}[{\includegraphics[width=1in,height=1.25in,clip,keepaspectratio]{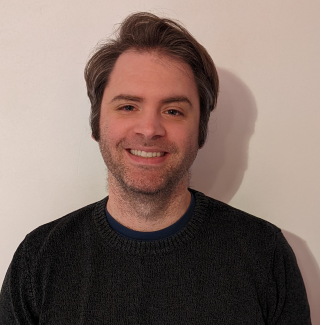}}]{Kristian Fenech}
Kristian Fenech received a Ph.D. degree in Physics from Swinburne University of Technology, Melbourne, Australia, in 2016.
He is currently an assistant professor with the Department of Artificial Intelligence, Faculty of Informatics, Eötvös Loránd University, Budapest, Hungary. His main research interests include machine learning, social signal processing, human centered AI and human-machine interaction.
\end{IEEEbiography}
\vspace{-1cm}
\begin{IEEEbiography}[{\includegraphics[width=1in,height=1.25in,clip,keepaspectratio]{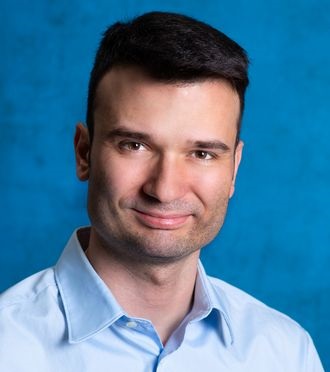}}]{P\'{e}ter Kovács} received the Ph.D. degree in computer
science from Eötvös Loránd University (ELTE),
Budapest, Hungary, in 2016; and completed his habilitation in computer science in 2022; since then, he has been an associate professor at the Department of Numerical Analysis of ELTE in Budapest, Hungary. In 2012, he was a visiting researcher at the Department of Signal Processing, Tampere University of Technology in Finland. From 2018 to 2020, he was a postdoc at the Institute of Signal Processing, Johannes Kepler University Linz in Austria. In 2016, he was honored with the Farkas Gyula Prize in applied mathematics from the János Bolyai Mathematical Society. With over 40 published scientific papers, his work on Variable Projection Networks was recognized with the Hojjat Adeli Award by the International Journal of Neural Systems. He is a member of the public body of the Hungarian Academy of Sciences and the IEEE EMBS Young Professionals Committee. His main research interests include signal and image processing, numerical analysis, system identification, and explainable AI. 
\end{IEEEbiography}


\begin{IEEEbiography}[{\includegraphics[width=1in,height=1.25in,clip,keepaspectratio]{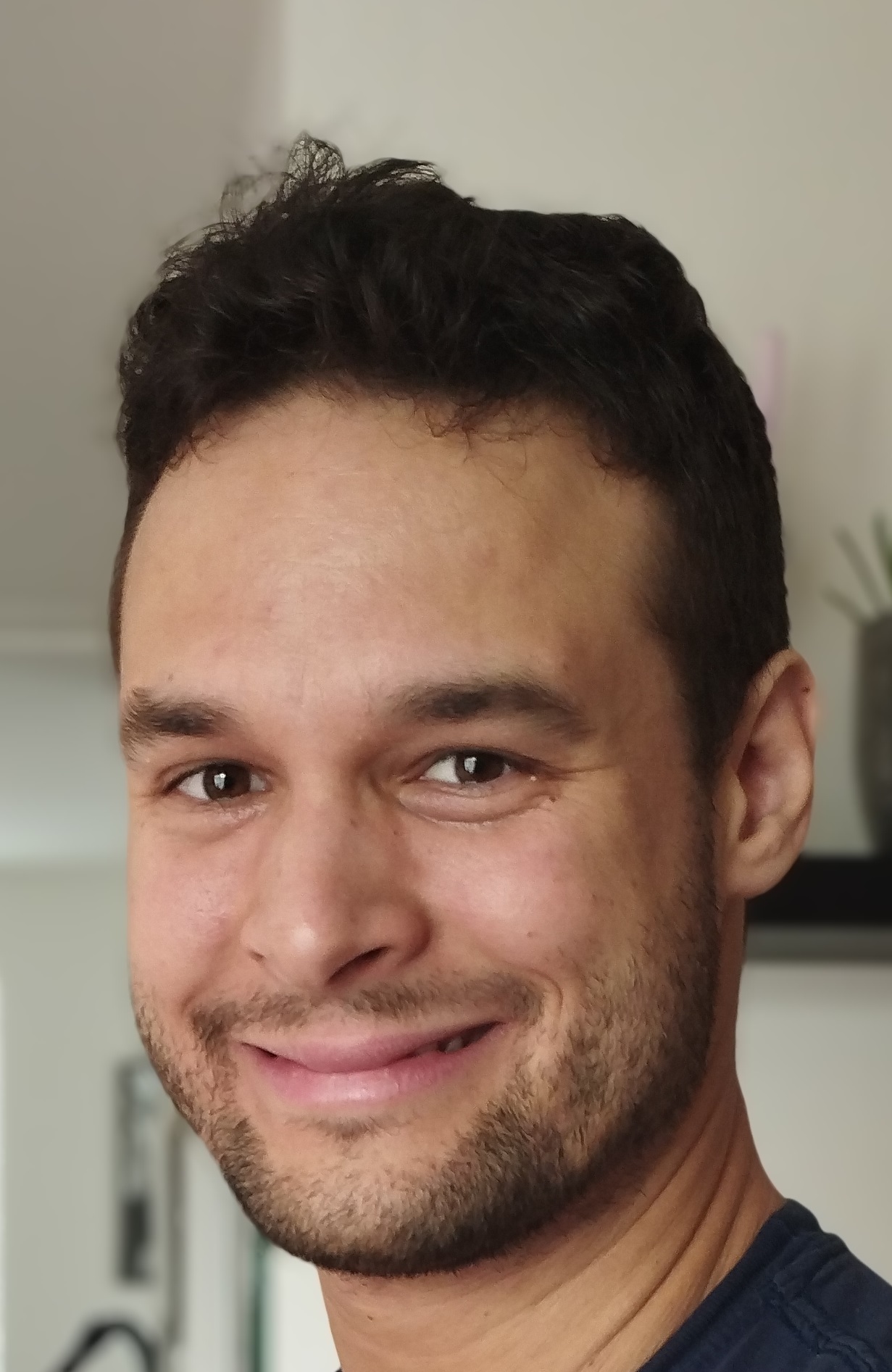}}]{Tamás Dózsa} received the M.Sc. degree in computer
science from Eötvös Loránd University (ELTE),
Budapest, Hungary, in 2020, where he is currently
pursuing the Ph.D. degree in computer science. During his B.Sc. and M.Sc. studies he received numerous recognitions at the university and national level for participating in scientific competitions. He authored 20 scientific publications thus far.

Since 2020, he has been working as a Research
Associate with the HUN-REN Institute for Computer Science and Control (SZTAKI) Systems and Control Laboratory, Budapest. His research interests
include numerical analysis, signal processing, optimization, and system theory.
\end{IEEEbiography}




\end{document}